\documentclass{article}

\usepackage{arxiv}
\usepackage{microtype}
\usepackage{graphicx}
\usepackage{subfigure}
\usepackage{booktabs} 

\usepackage{xcolor}
\definecolor{compblue}{rgb}{0,0.45,1}
\definecolor{darkblue}{rgb}{0,0,0.55}

\usepackage{hyperref}
\hypersetup{citecolor=darkblue, urlcolor=compblue, linkcolor=darkblue, colorlinks=true}

\usepackage{amsmath}
\usepackage{amssymb}
\usepackage{mathtools}
\usepackage{amsthm}
\usepackage{accents}

\usepackage{natbib}

\usepackage[capitalize,noabbrev]{cleveref}

\usepackage{thmtools}
\usepackage{thm-restate}
\theoremstyle{plain}
\newtheorem{theorem}{Theorem}[section]
\declaretheorem[name=Proposition, numberlike=theorem]{proposition}
\newtheorem{lemma}[theorem]{Lemma}

\theoremstyle{definition}
\newtheorem{definition}[theorem]{Definition}
\newtheorem{assumption}[theorem]{Assumption}
\theoremstyle{remark}

\crefname{assumption}{Assumption}{Assumptions}

\usepackage[textsize=tiny]{todonotes}
\usepackage{authblk}
\author[1]{\textbf{Samuel Lippl}}
\author[1]{\textbf{L.F. Abbott}}
\author[1,2]{\textbf{SueYeon Chung}}
\affil[1]{Department of Neuroscience, Zuckerman Mind Brain Behavior Institute, Columbia University, New York, NY}
\affil[2]{Center for Computational Neuroscience, Flatiron Institute, New York, NY}
\affil[ ]{\texttt{\{sl4742,lfabbott\}@columbia.edu,schung@flatironinstitute.org}}

\title{The Implicit Bias of Gradient Descent on\\Generalized Gated Linear Networks}

\renewcommand{\headeright}{}
\renewcommand{\undertitle}{}
\renewcommand{\shorttitle}{The Implicit Bias of Gradient Descent on Generalized Gated Linear Networks}

\begin{document}
\maketitle

\begin{abstract}
	Understanding the asymptotic behavior of gradient-descent training of deep neural networks is essential for revealing inductive biases and improving network performance.
We derive the infinite-time training limit of a mathematically tractable class of deep nonlinear neural networks, gated linear networks (GLNs), and generalize these results to gated networks described by general homogeneous polynomials.
We study the implications of our results, focusing first on two-layer GLNs.
We then apply our theoretical predictions to GLNs trained on MNIST and show how architectural constraints and the implicit bias of gradient descent affect performance.
Finally, we show that our theory captures a substantial portion of the inductive bias of ReLU networks.
By making the inductive bias explicit, our framework is poised to inform the development of more efficient, biologically plausible, and robust learning algorithms.
\end{abstract}

\section{Introduction}

Even after a task has been learned perfectly over a training set, learning typically continues to modify network parameters, often indefinitely.
Understanding the form of the infinite-time, asymptotic solution has important implications for understanding the inductive bias of the network and improving its learning algorithm.
For linear networks, $y=\langle\beta,x\rangle$, \citet{soudry_implicit_2018} have shown that gradient descent for most common losses used for classification asymptotically approaches a fixed margin classifier $\beta$ with minimum $L_2$ norm (or equivalently a maximum margin classifier with fixed norm, that is, an SVM).
\citet{gunasekar_implicit_2018} extended this analysis to deep linear networks, showing that densely connected linear networks converge to the SVM solution regardless of their depth.

The results of \citet{soudry_implicit_2018} and \citet{gunasekar_implicit_2018} have been generalized to homogeneous nonlinear predictors by \citet{nacson_lexicographic_2019} and \citet{lyu_gradient_2020}.
They demonstrate that gradient descent converges to a fixed margin classifier that minimizes the $L_2$ norm over all weights parameterizing the network.
It remains unclear, however, what effect this penalty on the weights has on the inductive bias of the network they parameterize.

To answer this question, we extend the approach by \citet{gunasekar_implicit_2018} to a nonlinear class of networks, in which $\beta$ is a set of different predictors and the linear predictor for a specific input is selected by a context system.
We assume that $\beta$ is constructed from homogeneous polynomials of a global pool of weights $w$ and that the network learns by gradient descent on $w$.
Two interesting examples of such networks are gated linear networks \citep[GLNs;][]{veness_online_2017} and what we call frozen-gate ReLU networks (FReLUs).
Because we focus on families of gated predictors, our approach covers certain models that escape the analysis by \citet{nacson_lexicographic_2019} and \citet{lyu_gradient_2020}.

In this paper, we present several key findings.
\begin{itemize}
    \item We derive and prove that gradient descent in GLNs asymptotically constructs a fixed margin solution with minimum norm on the collection of context-dependent linear predictors $\beta$, but with two important modifications: because all predictors are constructed from a shared set of weights, the norm of these vectors is minimized subject to certain equivariance constraints. For the same reason, gradient descent operates on a set of weights that are activated for multiple contexts and, as a consequence, the norm that $\beta$ minimizes is different from the usual $L_2$ norm of an SVM.
    \item Unlike prior work, we use the asymptotic limit of gradient descent to directly train SVMs that share the GLNs' inductive bias. These models allow us to separately study how the equivariance constraints and gradient descent's implicit bias affect generalization.
    \item We then analyse the implicit bias of these GLNs in detail. Our analysis highlights that gradient descent (without any explicit regularization) incentivizes higher similarity between predictors that share part of their context and that this improves the GLN's performance.
    \item Finally, we investigate the asymptotic learning behavior of ReLU networks by applying a similar approach to FReLUs. We show that ReLU networks outperform GLNs in part because they can modify their context throughout learning, whereas GLNs cannot.
\end{itemize}

Taken together, our results have implications for understanding learning in deep neural networks and draw a connection between deep learning and SVMs that may inspire new and improved learning algorithms.

\section{Gated Linear Networks}
\label{sec:glns}

GLNs, like other deep neural networks, process their input through multiple layers of hidden units to compute their output.
However, while conventional deep neural networks attain expressivity through nonlinearities, a hidden unit in a GLN computes its output without any nonlinearity.
Instead, it attains expressivity by using different weights for different regions of the input space.\footnote{\citet{veness_online_2017} motivate GLNs through opinion pooling. We omit this motivation in favor of a simpler (but equivalent) description of the architecture.
The motivation through opinion pooling also meant that they required their input to be scaled to $(0,1)$, which usually amounted to a squashed version of the input \citep{veness_gated_2021}.
Since opinion pooling subsequently expands these probabilities using the inverse sigmoid, this does not tend to be very different from the way we set up our GLNs.}

Consider an example: a two-dimensional input (\cref{fig:gln-example}d) with two units in the first hidden layer (\cref{fig:gln-example}a).
The first unit may use one set of weights to compute its output for $x_1>0$, and a different set of weights to compute its output for $x_1\leq 0$ (as reflected by the two regions in \cref{fig:gln-example}b and the two hyperplanes in \cref{fig:gln-example}e).
Similarly, the second unit may choose its set of weights differently for $x_2>0$ and $x_2\leq 0$ (\cref{fig:gln-example}c,f).
As we combine multiple such hidden units, the GLN becomes increasingly expressive.
For example, if a single unit now reads out the two hidden units, this readout will have learned a different linear predictor for all four quadrants (as reflected by the four hyperplanes in \cref{fig:gln-example}g).


More generally, a GLN learns a different weight vector, $\beta_{\gamma}$, for each global context $\gamma$.
Most inputs will not share this global context, but will have the same \emph{local} context for particular hidden units.
This means that the GLN uses, and therefore updates, overlapping sets of weights.
For example, the bottom-right and top-right quadrant in \cref{fig:gln-example} share the local context of the first hidden unit and therefore use and update the same weights for this unit.
(Moreover, all units in this example use the same set of weights for the second layer.)
Compared to a shallow GLN using the same partitioning of the input space (which will activate a non-overlapping set of weights for each context), we will see that a deep GLN's linear predictors are equivariant with respect to the local contexts they share, imposing architectural constraints.
These constraints change the network's inductive bias by reducing the search space.
In addition, we will see that this equivariance changes the implicit bias of gradient descent.

\begin{figure*}[t]
\vskip 0.2in
\begin{center}
\centerline{\includegraphics[width=\textwidth]{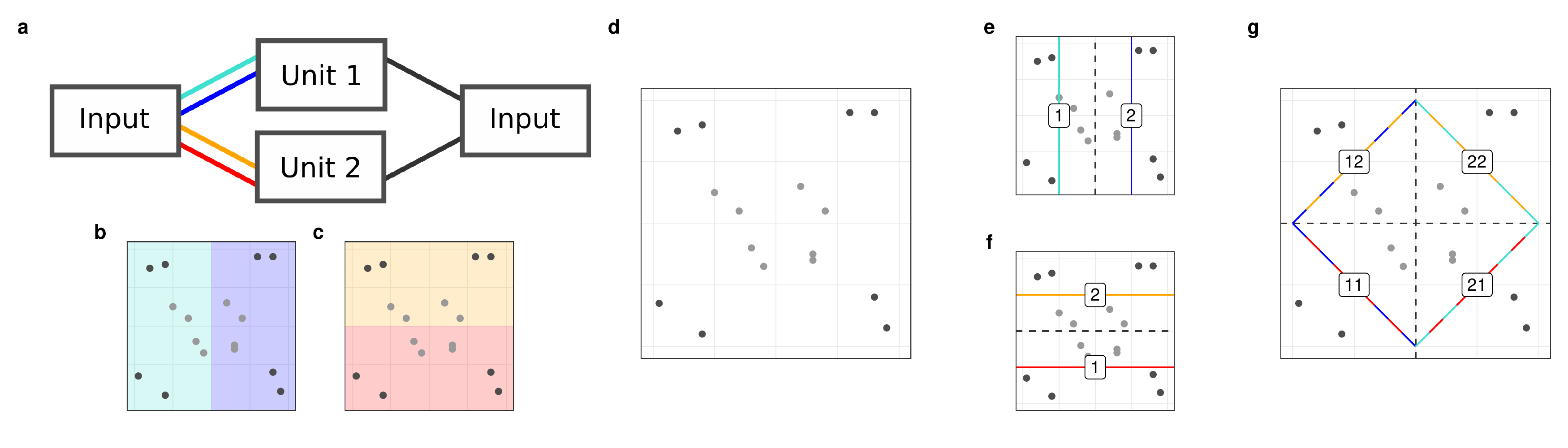}}
\caption{Example GLN. \textbf{a} Sketch of the GLN's structure. Every line represents a weight and units are linear. The colors indicate different weights, which are chosen depending on the region in which the input is located, as illustrated by panels \textbf{b} and \textbf{c}. For simplicity, we only use context-dependent weights for inputs to (not outputs from) the hidden layer. \textbf{d} The input data with labels indicated by shading. \textbf{e}, \textbf{f} The two hidden units each learn two weight vectors (as represented by the resulting hyperplanes). The weight is chosen depending on the context. \textbf{g} Output function. The output combines both hidden units and therefore has  weights in four different regions. Each weight (as represented by the multicolored hyperplanes) is composed from two of the hidden units' linear weights.}
\label{fig:gln-example}
\end{center}
\vskip -0.2in
\end{figure*}

This means that GLNs allow us to ask a fundamental question about gradient descent in deep networks: how do local changes in the weights affect the inductive bias of the global network parameterized by these weights?
GLNs are conventionally trained using a local learning rule, where every hidden unit attempts to predict the output.
However, we are interested in them specifically because their particular parameterization allows us to exactly characterize their asymptotic behavior under gradient descent.

\section{Exact Asymptotic Behavior of Learning in Gated Linear Networks}

\subsection{Background: Learning in Linear Networks}

To characterize the asymptotic behavior of learning in GLNs, we first turn to the asymptotic behavior of learning in linear networks, as characterized by \citet{soudry_implicit_2018} and \citet{gunasekar_implicit_2018}.
\citeauthor{soudry_implicit_2018} were concerned with gradient descent on a linearly separable dataset $(x^{(n)},y^{(n)})$ using the \emph{exponential loss} $\exp(-y\langle\beta,x\rangle)$, or a loss that has a similar tail\footnote{We call this class of loss functions \emph{exponential-like}, see \cref{def:exponential-like}.} such as the cross-entropy $\ln(1 +\exp(-y\langle\beta,x\rangle))$.

If, through gradient descent, the overall loss approaches zero, the linear predictor's norm $\|\beta^{(t)}\|_2$ must diverge. Thus, all of the theorems we discuss (including ours) refer to the asymptotic direction that the weight vector converges to (provided it converges to a fixed direction).  In particular, if we define $\hat{\beta}$ as the unit vector pointing in the same direction as the asymptotically diverging vector $\beta^{(\infty)}$, \citet{soudry_implicit_2018} show that $\hat{\beta}$ is a maximum margin predictor, or equivalently, that it is proportional to the solution of the optimization problem
\begin{equation}
\label{eq:opt-soudry}
    \min\|\beta\|_2,\quad\text{s.t. }y^{(n)}\langle\beta,x^{(n)}\rangle\geq1.
\end{equation}
Using the Karush-Kuhn-Tucker (KKT) conditions \citep{karush_minima_1939, kuhn_nonlinear_1951}, this minimizing vector can be written as
\begin{equation}
\label{eq:stat-soudry}
   \hat{\beta}=\sum_{n\in S}\lambda_ny^{(n)}x^{(n)},\quad
    \lambda_n\geq 0.
\end{equation}
Here, $S$ is the set of data points for which the margin inequality is tight, i.e. $S=\left\{n|y^{(n)}\langle\beta,x^{(n)}\rangle=1\right\}$. These are the \emph{support vectors}.

\citet{gunasekar_implicit_2018} extend this result to deep linear networks for which the output can be written as $\langle\mathcal{P}(w),x\rangle$, where $\mathcal{P}$ is a polynomial mapping the weights $w\in\mathbb{R}^P$ onto a linear predictor $\beta\in\mathbb{R}^D$.
$P$ and $D$ are the number of weights and the input dimension, respectively.
For example, a densely connected linear network with two layers has $f_w(x)=w_2^Tw_1x$, so $\mathcal{P}(w)=w_2^Tw_1$.
They require that $\mathcal{P}$ is homogeneous, that is, $\mathcal{P}(\alpha w)=\alpha^{\nu}\mathcal{P}(w)$, where $\nu$ is the \emph{degree} of $\mathcal{P}$ (this excludes skip connections and bias units). For the example above, $\nu=2$.
They prove that if $w^{(t)}$ converges in direction to $\hat{w}$, $\hat{w}$ is proportional to a solution of the optimization problem
\begin{equation}
\label{eq:opt-gunasekar}
    \min\|w\|_2,\quad\text{s.t. }y^{(n)}\langle\mathcal{P}(w),x^{(n)}\rangle\geq1.
\end{equation}
Whereas (\ref{eq:opt-soudry}) directly minimizes the norm of the linear predictor, this problem penalizes the overall norm of internal weights that parameterize that predictor.
The fixed margin constraint, however, still operates on the linear predictor $\mathcal{P}(w)$.

In contrast to the linear predictor, $\hat{w}$ is not necessarily a global minimum of (\ref{eq:opt-gunasekar}).
Instead, stationarity is akin to a local minimum in the context of minimizing a (potentially nonconvex) objective function, but additionally takes into account the margin constraints.
More specifically, stationarity requires that
\begin{equation}
\label{eq:stat-gunasekar}
    \hat{w} = \nabla_w\mathcal{P}(w)\sum_{n\in S}\lambda_ny_nx_n,\quad
    \lambda_n\geq 0,
\end{equation}
so the weights are still constructed from a nonnegative sum of support vectors.
However, these support vectors must be projected from the input space $\mathbb{R}^D$ to the weight space $\mathbb{R}^P$.
This is achieved by the polynomial's Jacobian $\nabla_w\mathcal{P}(w)\in\mathbb{R}^{P\times D}$.

\subsection{Extension to Generalized GLNs}

In contrast to linear predictors and multi-layer linear networks, which are characterized by a single linear predictor, GLNs are characterized by a different weight vector $\beta_{\gamma}$ for each global context $\gamma$. Each $\beta_{\gamma}$ is given by a polynomial function of the weights $w$, and, if we leave out bias units beyond the first layer, this polynomial is homogeneous. Thus, it may appear that we can extend the analysis by \citet{gunasekar_implicit_2018} simply by applying their theorem to each context-specific predictor $\beta_{\gamma}$. However, predictors for different contexts are parameterized by an overlapping set of weights, so this is not possible. This highlights the critical question we pose in our analysis: \emph{how do the shared weights between different linear predictors affect the inductive bias of gradient descent?}

Motivated by these considerations, we extend the previous analysis by considering a set of (global) contexts $\gamma\in\Gamma$, where each context uses a different homogeneous polynomial $\mathcal{P}_{\gamma}$ to connect the global pool of weights to the context-specific linear predictor.
This means that our model
\begin{equation}
    f_w(x;\gamma):=\langle\mathcal{P}_{\gamma}(w),x\rangle,
\end{equation}
depends on both the input $x$ and the context $\gamma$.
We call this class of functions \emph{Generalized Gated Linear Networks} and, in particular, it covers GLNs (without bias units).

We are then able to prove the following (see \cref{app:proof-th} for assumptions and proof):
\begin{restatable}{theorem}{ghpt}
\label{th}
Consider a dataset $(x^{(n)},y^{(n)},\gamma^{(n)})$, where $x^{(n)}\in\mathbb{R}^D$ is the input, $y^{(n)}=\{-1,1\}$ is the label, and $\gamma^{(n)}$ is a global context. Then if $w^{(t)}$ converges to a fixed direction and the loss approaches zero, the limiting direction $\hat{w}$ is proportional to a stationary point of
\begin{equation}
\label{eq:th}
    \min\|w\|_2^2,\quad\text{s.t. }y^{(n)}f_w(x^{(n)};\gamma^{(n)})\geq1.
\end{equation}
\end{restatable}
Here and throughout the article, $\|w\|_2^2$ denotes the sum of the squares of all the elements in all of the weight matrices of the network.
This stationary point $\hat{w}$ is given by
\begin{equation}
\label{eq:th-stat}
    \hat{w}=\sum_{\gamma\in\Gamma}\nabla_w\mathcal{P}_{\gamma}(w)\sum_{n\in S_{\gamma}}\lambda_ny_nx_n,\quad
    \lambda_n\geq 0,
\end{equation}
where $S_{\gamma}$ is the context-specific set of support vectors. Just as in (\ref{eq:stat-gunasekar}), we sum over this set of support vectors and project it into the weight space $\mathbb{R}^P$.
The stationary point $\hat{w}$ is then given by the sum of these contextwise projections.

\subsection{Proof Sketch}
We provide here an outline of the proof of \cref{th}.\footnote{A rigorous proof can be found in \cref{app:proof-th}.}
Because this theorem is a relatively straightforward extension of the theorems by \citet{soudry_implicit_2018} and \citet{gunasekar_implicit_2018}, we begin with an outline of their proofs.

\subsubsection{Background: Sketch of Previous Proofs}

\citet{soudry_implicit_2018} consider the loss function
\begin{equation}
\label{eq:exp-loss}
    \mathcal{L}(\beta)=\sum_{n=1}^N\exp(-y_n\langle\beta,x_n\rangle),
\end{equation}
and gradient descent updates $-\eta_t\nabla_{\beta}\mathcal{L}(\beta)$,
where $\eta_t>0$ is the learning rate.
They then rely on two facts: that the gradient descent updates converge to some limit direction and that early updates are eventually forgotten. These imply that for large $t$, the weight direction approaches the limiting direction of the gradient descent updates.

Because
\begin{equation}
    -\nabla_{\beta}\mathcal{L}(\beta)=\sum_{n=1}^N\exp(-y_n\langle\beta,x_n\rangle)y_nx_n,
\end{equation}
the weight $\exp(-\|\beta^{(t)}\|_2y_n\langle\hat{\beta},x_n\rangle)$ will converge to zero as $\|\beta^{(t)}\|_2$ increases.
However, the weights of the data points with the smallest margins, i.e. the support vectors, will converge to zero exponentially slower than all other data points.
Thus, the support vectors dominate the gradient's direction and we can write
\begin{equation}
    \hat{\beta}=\sum_{n\in S}\lambda_ny_nx_n,
\end{equation}
which is (\ref{eq:stat-soudry}). The individual values for $\lambda_n$ are determined by the rate with which the loss of individual data points, $\exp(-y_n\langle\beta,x_n\rangle)$, approaches zero.

\citet{gunasekar_implicit_2018} extend this result by decomposing the loss gradient for polynomial predictors into $\nabla_w\mathcal{L}(w)=\nabla_w\mathcal{P}(w)(-\nabla_{\mathcal{P}(w)}\mathcal{L}(w))$.
The latter part, $-\nabla_{\mathcal{P}(w)}\mathcal{L}(w)$, corresponds to the gradient in the linear predictor.
Even though gradient descent is not performed on this linear predictor directly, the result of \citet{soudry_implicit_2018} generalizes: if $-\nabla_{\mathcal{P}(w)}\mathcal{L}(w)$ converges to some limit direction, we can\footnote{\citet{gunasekar_characterizing_2018} prove this for the exponential loss and note that they expect the result to generalize to exponential-like losses. We prove this generalization in \cref{app:extension-gunasekar}.} again infer that the support vectors will, at large times, dominate this direction. This implies that we can approximate the loss gradient (and thus the weight directional limit) as a sum of support vectors that is projected into the weight space by the Jacobian, as is given by (\ref{eq:stat-gunasekar}).

\subsubsection{Proof Sketch of \cref{th}}

To extend these results to GLNs, we decompose the loss function into a sum over the losses specific to each context,
\begin{equation}
    \mathcal{L}_{\gamma}(w)=\sum_{n:\gamma^{(n)}=\gamma}\exp(-y_n\langle\mathcal{P}_{\gamma}(w),x_n\rangle).
\end{equation}
The gradient can be decomposed similarly into
\begin{equation}
    \nabla_w\mathcal{L}(w)=\sum_{\gamma}\nabla_w\mathcal{L}_{\gamma}(w).
\end{equation}
Extending the strategy of \citet{gunasekar_implicit_2018} to a sum of loss functions, we asymptotically express $w^{(t)}$ as a weighted sum of the individual limit directions,
\begin{equation}
    w^{(t)}\approx \|w^{(t)}\|_2\sum_{\gamma}\xi_{\gamma}\nabla_w\mathcal{P}_{\gamma}(w)(-\nabla_{\mathcal{P}_{\gamma}(w)}\mathcal{L}_{\gamma}(w)).
\end{equation}
(The linear weights $\xi_{\gamma}$ are necessary because the different components of the loss might be scaled differently.)
The contextwise gradients are again dominated by the support vectors and we can absorb $\xi_{\gamma}$ into $\lambda_n$ to arrive at (\ref{eq:th-stat}).

\section{Gated Linear Networks with Two Layers}

\cref{th} allows us to characterize the implicit bias of gradient descent.  However, the minimized norm in this theorem is that of the weights parameterizing the context-dependent linear predictors, not the predictors themselves.  To connect our results directly to the linear predictors, we consider a special case: GLNs of depth 2, with one output neuron and one context for this output neuron. This means that we have two hyperparameters for our architecture: the number of hidden units $H$ in the first layer, and the number of contexts per hidden unit $C$.
As a consequence, the global context is given by $\gamma\in\{1,\dotsc,C\}^H$.
The GLN's weights are given by
\begin{equation}
    w=(w^{(1)},w^{(2)}),\quad
    w^{(1)}\in\mathbb{R}^{H\times C\times D},
    w^{(2)}\in\mathbb{R}^H,
\end{equation}
and the resulting linear predictors are
\begin{equation}
\label{eq:beta-param}
    \beta_{\gamma}=\sum_{h=1}^Hw^{(2)}_hw^{(1)}_{h\gamma_h}\in\mathbb{R}^D.
\end{equation}
Throughout this exposition, we consider as a simple example $H=C=2$, as in \cref{fig:gln-example}.

As we noted in \cref{sec:glns} and are now able to analyse in more detail, the parameterization of these networks affects their inductive bias in two ways.
First, it imposes \emph{architectural constraints} on the resulting linear predictors $\beta_{\gamma}$.
This manifests in an equivariance condition on neighboring predictors: the difference between two linear predictors is invariant to the contexts they share.
In the case of $C=H=2$, this condition is given by
\begin{equation}
    \beta_{21}-\beta_{22}=\beta_{11}-\beta_{12}.
    \label{eq:beta-id}
\end{equation}
As we can see, the two predictors on the left share the first unit's local context and so changing this context does not affect their difference.
More generally (see \cref{app:special-case-arch}), this means that even though $\beta$ specifies a set of $C^H$ linear predictors, we can only choose $(C-1)H+1$ of them freely.

Second, whereas shallow networks minimize the $L_2$ norm, gradient descent on deep GLNs implicitly minimizes a different norm, which we call the \emph{GLN norm} and denote by $\|\cdot\|_{\rm GLN}$.
Importantly, this norm operates on the linear predictors $\beta_{\gamma}$ instead of the underlying global pool of weights.
We characterize $\|\cdot\|_{\rm GLN}$ in \cref{sec:gln-norm}.
First, however, we would like to illustrate why it is important to understand the difference between $\|\cdot\|_2$ and $\|\cdot\|_{\rm GLN}$.
To this end, the next section illustrates that the $\|\cdot\|_{\rm GLN}$ is not only more consistent with a GLN trained with gradient descent, but also leads to better generalization on MNIST.

\subsection{The GLN-Norm Improves Generalization Over the $L_2$ Norm}

\begin{figure*}[t]
\vskip 0.2in
\begin{center}
\centerline{\includegraphics[width=\textwidth]{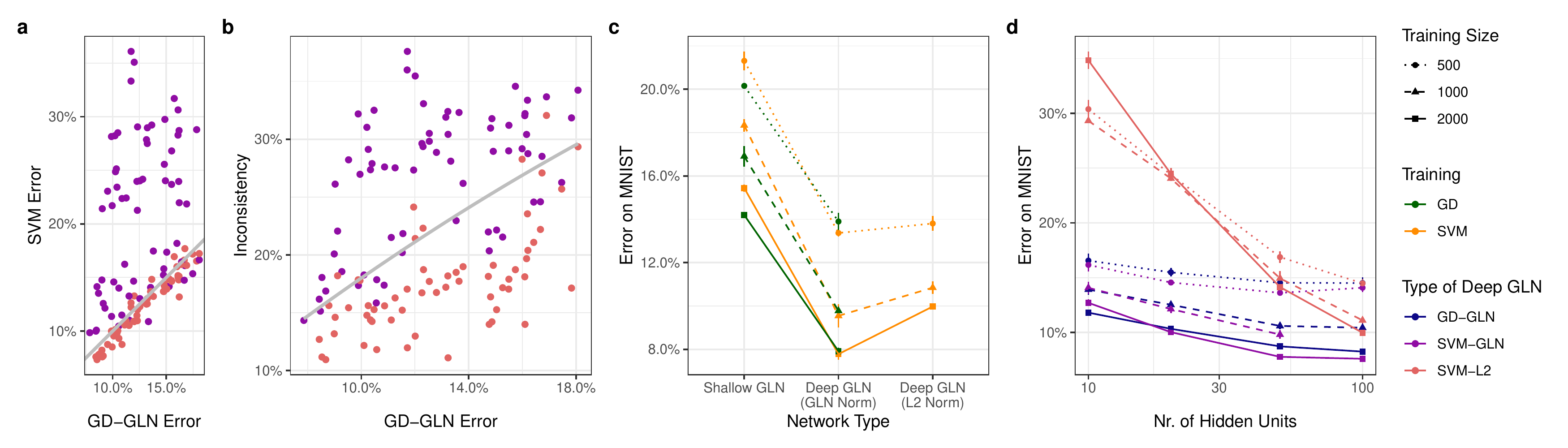}}
\caption{Experiments on GLNs. \textbf{a} The error of the GD-GLN plotted against the error of the SVMs. The grey line represents identical performance. \textbf{b} Inconsistency between the GD-GLN and the SVM plotted against the GD-GLN's error on MNIST. The grey line represents the inconsistency we would expect from a predictor with matched error rate, but no further correlation with the network (see \cref{fn:inconsistency}). \textbf{c} Performance of the best shallow and deep GLNs across the number of contexts and (in the case of deep GLNs) hidden units. \textbf{d} Error of the SVMs and the GD-GLN plotted against the number of hidden units. This plot depicts the networks with two contexts per hidden unit, see \cref{fig:supp-gln} for the networks with four contexts per hidden unit.}
\label{fig:exp-gln}
\end{center}
\vskip -0.2in
\end{figure*}

Using our theorem, we can examine the impact of the architectural constraints on the deep GLN alone or together with the resulting changes in implicit bias.
We fit two support vector machines that respect the deep GLN's architectural constraints, minimizing either the $L_2$ norm (SVM-L2) or the GLN norm (SVM-GLN) while maintaining a fixed margin.  Note that it is the second of these that reflects the full result of our theorem.
Both optimization problems are convex, so we can use convex optimization algorithms \citep{diamond_cvxpy_2016} and are guaranteed to find a global minimum.
In addition, we trained a GLN using gradient descent (GD-GLN) for 3200 steps in PyTorch \citep{paszke_pytorch_2019}.

We trained all our models on MNIST \citep{lecun_mnist_2010}, which we turn into a binary classification problem by grouping together the digits 0-4 and 5-9.
Since we use full-batch convex optimization, we are restricted in the size of our training data, using subsets of 500, 1000, and 2000 data points.
To evaluate generalization, we use a validation dataset with 12000 data points.

We consider GLNs with 10, 20, 50, and 100 hidden units and two or four contexts per hidden unit.
We assigned contexts by partitioning the input space using randomly sampled hyperplanes (as is illustrated in \cref{fig:gln-example}), similar to \citet{veness_online_2017}.
For every architecture, we used three random seeds to sample these hyperplanes.
\cref{fig:exp-gln} depicts the mean and standard deviation across these three runs.\footnote{More details on the experimental setup can be found in \cref{app:experimental-setup}. Code to reproduce all experiments can be found at \url{https://github.com/sflippl/implicit-bias-glns}.}

\cref{fig:exp-gln}a compares the error of the GD-GLN with the SVMs that use the same contexts, number of hidden units, and number of contexts per hidden units.
SVM-GLN matches the GD-GLN in accuracy and even outperforms it slightly.
In contrast, SVM-L2 performs much worse and its performance is only weakly correlated with that of the GD-GLN with matching hyperparameters.

\cref{fig:exp-gln}b depicts the proportion of the validation data for which the SVM predicts different labels than the GD-GLN (inconsistency).
If the GD-GLN had truly converged to the SVM-GLN, the inconsistency would be zero.
Instead the two predictors make inconsistent predictions on a substantial proportion (more than 10\%) of the data.
The inconsistency tends to be lower than we would expect from two models with a matching error rate, but no further correlation (grey line).\footnote{Suppose this error rate is $p$. Inconsistent labels mean that one model makes an error and the other does not. The probability of this happening is $2p(1-p)$.\label{fn:inconsistency}}
It is also much lower than that between the SVM-L2 and the GD-GLN.
Still, this result highlights a substantial difference between the infinite-time predictor we derived and its finite-time counterpart (see Discussion).

Next, we looked at how different choices of depth and width interact with the implicit bias of gradient descent.
\cref{fig:exp-gln}c depicts the best-performing shallow and deep GLN (i.e. one and two layers) across all hyperparameters.
This illustrates that the architectural constraints paired with the $L_2$ norm already allow the deep GLN to find better solutions than the shallow GLN (which simply learns an SVM for each context).
However, we again see that the SVM-GLN further improves in performance over the SVM-L2.
Isolating the effects that making the network deeper has on the functions it can express, and on the solutions that gradient descent discovers in practice, would not have been possible without our theory.

Finally, \cref{fig:exp-gln}d illustrates that the SVM-L2 performs much worse than the SVM-GLN for few hidden units in particular.
This panel also makes particularly apparent that the SVM-GLN tends to slightly outperform the GD-GLN with the same hyperparameters.
This is exactly what we would expect if the GD-GLN slowly converges to the SVM-GLN and if its generalization performance improves throughout infinite training.

\subsection{Understanding the GLN Norm}
\label{sec:gln-norm}

Having seen that it provides a useful inductive bias, we now turn to understanding the GLN norm.
(\ref{eq:beta-param}) makes apparent that $w$ only affects $\beta$ through the auxiliary variable
\begin{equation}
    \zeta_{h\gamma_h}:=w^{(2)}_hw^{(1)}_{h\gamma_h}.
\end{equation}
For a fixed $\zeta$, which choices of $w$ minimize $\|w\|_2^2$?
Intuitively, the $L_2$ norm incentivizes us to distribute magnitudes across parameters equally.
Since $\zeta_h=w^{(2)}_hw^{(1)}_h$, the two parameters should share the magnitude $\|\zeta_h\|_2$ equally, i.e.
\begin{equation}
    |w^{(2)}_h|=\|w^{(1)}_h\|_2=\sqrt{\|\zeta_h\|_2}.
\end{equation}
This implies\footnote{Technically, we have to check equivalence of the KKT conditions. We do so in \cref{app:special-case} and the same intuition applies.} that
\begin{equation}
    \|w\|_2^2=\sum_{h=1}^H\|w^{(1)}_h\|_2^2+|w^{(2)}_h|^2\propto\sum_{h=1}^H\|\zeta_h\|_2,
\end{equation}
and thus
\begin{equation}
\label{eq:beta-gln}
    \|\beta\|_{\rm GLN}=\min_{\zeta}\sum_{h=1}^H\|\zeta_h\|_2,\quad
    \text{s.t. }\beta_{\gamma}=\sum_{h=1}^H\zeta_{h\gamma_h}.
\end{equation}
This means that, when expressed in terms of $\zeta$, the GLN norm takes on the form of a group lasso \citep{yuan_model_2006}.
Importantly, this norm is different from the $L_2$ norm on $\zeta$, which would involve summing up $\|\zeta_h\|_2^2$ instead of $\|\zeta_h\|_2$.
Because it computes the $L_2$ norm of $\zeta_h$, the GLN norm encourages the magnitude of this vector to be as small as possible.
But because it sums up the norm itself instead of its square, it also incentivizes setting entire components $\zeta_h$ to zero, similar to how the L1 norm incentivizes setting single entries of a vector to zero.
Put differently, the GLN norm incentivizes sparsity in the components $\zeta_h$.
Because each component $\zeta_h$ encodes differences in the predictor as a consequence of the different local contexts of the hidden unit $h$, this norm therefore encourages the set of linear predictors to only learn differences between unit-specific contexts if this is actually useful.

We can further illustrate the difference between $\|\cdot\|_{\rm GLN}$ and the $L_2$ norm by considering the special case $H=C=2$.
In this case,
\begin{equation}
\label{eq:beta-norm-2}
    \|\beta\|_{\rm GLN}^2=\|\beta\|_2^2+\tfrac12\sum_{i,j}\|\beta_{ij}-\beta_{\overline{i}j}\|_2\|\beta_{ij}-\beta_{i\overline{j}}\|_2,
\end{equation}
where $\overline{k}$ denotes the local context opposite to $k$, i.e. $\overline{k}=2$ if $k=1$ and $\overline{k}=1$ if $k=2$.  
The GLN norm therefore adds to the $L_2$ norm a component that encourages neighboring predictors to be more similar.
Without any explicit regularization, the equivariant interactions of the GLN cause predictors that share parts of their global context (and thus overlap in their weights) to become more similar to each other.

\section{Frozen-Gate ReLU Networks}

While we were motivated by understanding GLNs, \cref{th} applies to other architectures as well.
We apply the theorem to a particular variation on ReLU networks that makes them generalized gated linear networks: frozen-gate ReLU networks.

A single hidden unit in a ReLU network computes its activation as $z=\max(\langle w,x\rangle,0)$.
That is, it first computes a linear function and then sets any negative values to zero.
For fixed weights, we can also implement this with a gated linear predictor.
More specifically, there are two contexts associated with the hidden unit, depending on the sign of $\langle w,x\rangle$.
If $\langle w,x\rangle >0$, we use the weight $w$ to compute the hidden unit.
If $\langle w,x\rangle\leq 0$, we instead use a zero vector.
The strategy of separating the gates in this way is similar to \citet{lakshminarayanan_neural_2020}, who use it to define a neural tangent kernel.

For fixed weights, this gated linear predictor is exactly equivalent to the ReLU network.
However, we train it by only changing the linear weights, freezing the gates that determine the context for each hidden unit.
We thus call this architecture \emph{frozen-gate ReLU networks} (FReLUs).
Throughout gradient descent, as the weights change but the contexts remain fixed, the FReLU diverges from the standard ReLU network.  To mitigate this divergence, we also consider networks in which the weights are learned in the usual way for a period of time and then the gates are frozen to apply the asymptotic analysis.

Using FReLUs as an approximation, can our theory shed light on the inductive bias of ReLU networks and how it is different from that of GLNs?
To investigate this, we first characterize the implicit bias of FReLUs.
We then compare FReLUs with ReLU networks trained with gradient descent.

\subsection{The Implicit Bias of FReLUs}

FReLUs are structured almost like GLNs, except that one of the two context-gated weight vectors is set to zero.
It is therefore not surprising that they also minimize a group Lasso norm.
Specifically, for a given context $\gamma\in\{0,1\}^H$, we can parameterize the resulting linear predictor as $\beta_{\gamma}=\sum_{h:\gamma_h=1}\zeta_h$, where $\zeta\in\mathbb{R}^{H\times D}$ is an auxiliary variable defined similarly as in \cref{sec:gln-norm}.
FReLUs then minimize the norm
\begin{equation}
\label{eq:frelu-norm}
    \|\zeta\|_{\rm FReLU}:=\sum_{h=1}^H\|\zeta_h\|_2.
\end{equation}
Since the group Lasso encourages sparsity, this means that unless it would otherwise increase the network's margin, $\zeta_h$ will be low or set to zero.
Since the $\zeta_h$'s induce the kinks in the network's separating hypersurface, this means that gradient descent (again without any explicit regularization) encourages this surface to be as straight as possible.

\subsection{Comparing ReLU networks and FReLUs}

\begin{figure*}[t]
\vskip 0.2in
\begin{center}
\centerline{\includegraphics[width=\textwidth]{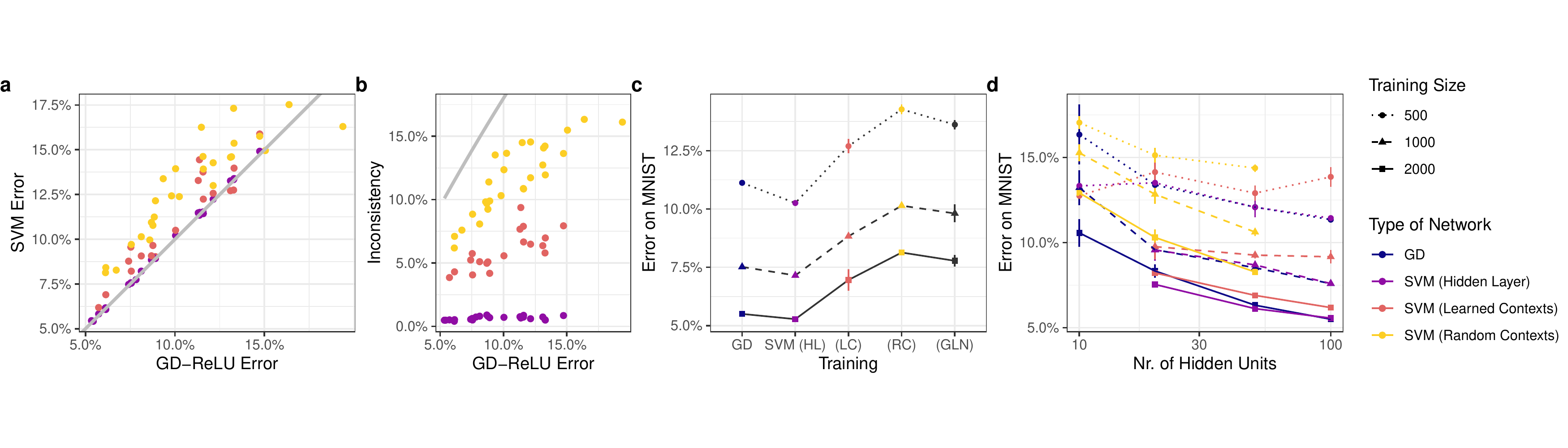}}
\caption{Experiments on ReLU networks. \textbf{a} The error of the GD-ReLU plotted against the SVMs. The grey line represents identical performance. \textbf{b} Inconsistency between the GD-ReLU and the SVMs plotted against the GD-ReLU's error on MNIST. The grey line represents the inconsistency we would expect from a predictor with matched error rate, but no further correlation with the network (see \cref{fn:inconsistency}). \textbf{c} Performance of the best ReLU networks as well as the best deep GLN. To make sure the deep GLN does not have more parameters than the FReLU, we only use those GLNs with two contexts and at most 50 hidden units. \textbf{d} Error of the SVMs and the GD-ReLU plotted against the number of hidden units.}
\label{fig:exp-relu}
\end{center}
\vskip -0.2in
\end{figure*}

Can we use this insight to better understand generalization in ReLU networks trained with gradient descent?
To investigate this, we trained a ReLU network on the binary MNIST task using gradient descent (GD-ReLU).
We trained networks with 10, 20, 50, and 100 hidden units and used three random seeds for initialization.
We then compared each network to an SVM trained on $\|\cdot\|_{\rm FReLU}$ using random contexts and a matching architecture (SVM-RC).
This network's performance is already reasonably correlated with that of the matching GD-ReLU (\cref{fig:exp-relu}a).
However, a substantial proportion of its predictions do not match that of the GD-ReLU (\cref{fig:exp-relu}b), although they are still more consistent than we would expect from an uncorrelated network with matching accuracy.

One reason that the SVM-RC may have a different inductive bias from the GD-ReLU is that the latter adapts its contexts.
To see if this can explain part of the disparity, we trained an SVM on $\|\cdot\|_{\rm FReLU}$ using the context that GD-ReLU has learned at the end of its training (SVM-LC).
Indeed, this predictor generalizes better (\cref{fig:exp-relu}a) and is considerably more consistent with the GD-ReLU (\cref{fig:exp-relu}b).
However, an SVM trained on the hidden layer's activation (SVM-HL) still performs better (it even outperformed the GD-ReLU, \cref{fig:exp-relu}c) and is more consistent with the GD-ReLU.
This is not surprising as the SVM-HL has many fewer free parameters.
Nevertheless, this highlights that there is a nontrivial disparity between SVM-LC and GD-ReLU.

The fact that a FReLU with learned contexts outperforms one with random contexts indicates that learnable contexts can be beneficial for a network.
In particular, \cref{fig:exp-relu}c demonstrates that the SVM-LC outperforms a deep GLN of similar size whereas a FReLU using random contexts does not (and is, in fact, slightly worse).
\cref{fig:exp-relu}d shows that the SVM-LC uniformly outperforms the SVM-RC for any number of hidden units.\footnote{Missing data points indicate that constraints could not be satisfied, or that the optimizer did not converge in the allocated number of iterations (see \cref{app:experimental-setup}).}
However, its disparity to the GD-ReLU and the SVM-HL increases with increasing latent dimensions.

Our analysis demonstrates that the SVM-LC captures a substantial portion of the ReLU network's inductive bias.
FReLUs therefore promise a new perspective on why ReLU networks generalize well: they can adapt their gating function and use the resulting contexts as sparsely as possible.
Still, the SVM-LC is also substantially different from the GD-ReLU.
This may be because of finite-time effects or because the fact that ReLU networks learn their weights and gates in an entangled manner changes their inductive bias.
We leave investigating this question (for example using the results by \citet{lyu_gradient_2020}) to future work.

\section{Discussion}

In this article, we characterized the asymptotic behavior of gradient-descent training of Generalized Gated Linear Networks.
We used this theory to exactly characterize the norm minimized by a deep GLN and confirmed that this allows us to train an SVM that captures its performance.
This allowed us to tease apart the contributions of architectural constraints and the implicit bias of gradient descent, demonstrating that the implicit bias is essential for good generalization.
We also confirmed that this allows us to capture a substantial portion of the inductive bias of ReLU networks, attributing part of their generalization performance to the fact that they (a) learn their contexts and (b) use them as sparsely as possible.
This suggests that we might be able to take inspiration from ReLU networks to devise context learning algorithms for GLNs.
Conversely, a perspective that decomposes gradient-descent training in ReLU networks into context and weight learning, may shed new light on their inductive bias.

Our experiments indicate our theory's potential to help us understand the inductive bias of deep neural networks.
To realize this potential, we must address the fact that the infinite-time deep GLN still makes substantially different predictions from its finite-time counterpart.
This may be due to finite-time effects \citep{arora_implicit_2019}.
Alternatively, gradient descent and convex optimization may have converged to different subsets of the stationary points characterized by our theory \citep{nacson_lexicographic_2019}.

Still, the fact that we can train an SVM that matches (and even outperforms) its corresponding deep GLN indicates that our theory allows us to successfully disentangle the particular optimization procedure used from the inductive bias it implements.
This means that we can consider alternative learning algorithms that find the same stationary points, but have other benefits, for example faster convergence, more efficient computations, or higher biological plausibility.
To this end, comparing the inductive bias of gradient descent to that of the local learning rule conventionally applied to GLNs (for instance using the results by \citet{ji_implicit_2019}) may help us design new local learning rules that generalize better.
Finally, our framework connects networks trained with gradient descent to SVMs, which have formal adversarial protections \citep{mangasarian_arbitrary-norm_1999,gentile_robustness_2003}.
This perspective may therefore allow us to learn more robust networks, either by imposing the results of infinite-time training or by changing the inductive bias.

\section*{Acknowledgements}

We thank David Clark, Tiberiu Tesileanu, and Jacob Portes for helpful comments on an earlier version of the manuscript. We thank David Clark and Elom Amematsro for helpful discussions. Research was supported by NSF NeuroNex Award (DBI-1707398), the Gatsby Charitable Foundation (GAT3708), and the Simons Collaboration for the Global Brain.

\bibliography{main}

\begin{thebibliography}{}

\bibitem[Agrawal et~al., 2018]{agrawal_rewriting_2018}
Agrawal, A., Verschueren, R., Diamond, S., and Boyd, S. (2018).
\newblock A rewriting system for convex optimization problems.
\newblock {\em Journal of Control and Decision}, 5(1):42--60.

\bibitem[Arora et~al., 2019]{arora_implicit_2019}
Arora, S., Cohen, N., Hu, W., and Luo, Y. (2019).
\newblock Implicit {Regularization} in {Deep} {Matrix} {Factorization}.
\newblock In Wallach, H., Larochelle, H., Beygelzimer, A., Alché-Buc, F.~d.,
  Fox, E., and Garnett, R., editors, {\em Advances in {Neural} {Information}
  {Processing} {Systems}}, volume~32. Curran Associates, Inc.

\bibitem[Diamond and Boyd, 2016]{diamond_cvxpy_2016}
Diamond, S. and Boyd, S. (2016).
\newblock {CVXPY}: {A} {Python}-embedded modeling language for convex
  optimization.
\newblock {\em Journal of Machine Learning Research}, 17(83):1--5.

\bibitem[Domahidi et~al., 2013]{domahidi_ecos_2013}
Domahidi, A., Chu, E., and Boyd, S. (2013).
\newblock {ECOS}: {An} {SOCP} solver for embedded systems.
\newblock In {\em European {Control} {Conference} ({ECC})}, pages 3071--3076.

\bibitem[Gentile, 2003]{gentile_robustness_2003}
Gentile, C. (2003).
\newblock The {Robustness} of the p-{Norm} {Algorithms}.
\newblock {\em Machine Learning}, 53(3):265--299.

\bibitem[Grant et~al., 2006]{grant_disciplined_2006}
Grant, M., Boyd, S., and Ye, Y. (2006).
\newblock Disciplined {Convex} {Programming}.
\newblock In Liberti, L. and Maculan, N., editors, {\em Global {Optimization}:
  {From} {Theory} to {Implementation}}, pages 155--210. Springer US, Boston,
  MA.

\bibitem[Gunasekar et~al., 2018a]{gunasekar_characterizing_2018}
Gunasekar, S., Lee, J., Soudry, D., and Srebro, N. (2018a).
\newblock Characterizing {Implicit} {Bias} in {Terms} of {Optimization}
  {Geometry}.
\newblock In Dy, J. and Krause, A., editors, {\em Proceedings of the 35th
  {International} {Conference} on {Machine} {Learning}}, volume~80 of {\em
  Proceedings of {Machine} {Learning} {Research}}, pages 1832--1841. PMLR.

\bibitem[Gunasekar et~al., 2018b]{gunasekar_implicit_2018}
Gunasekar, S., Lee, J.~D., Soudry, D., and Srebro, N. (2018b).
\newblock Implicit {Bias} of {Gradient} {Descent} on {Linear} {Convolutional}
  {Networks}.
\newblock In Bengio, S., Wallach, H., Larochelle, H., Grauman, K.,
  Cesa-Bianchi, N., and Garnett, R., editors, {\em Advances in {Neural}
  {Information} {Processing} {Systems}}, volume~31. Curran Associates, Inc.

\bibitem[He et~al., 2015]{he_delving_2015}
He, K., Zhang, X., Ren, S., and Sun, J. (2015).
\newblock Delving {Deep} into {Rectifiers}: {Surpassing} {Human}-{Level}
  {Performance} on {ImageNet} {Classification}.
\newblock In {\em Proceedings of the {IEEE} {International} {Conference} on
  {Computer} {Vision} ({ICCV})}.

\bibitem[Ji and Telgarsky, 2019]{ji_implicit_2019}
Ji, Z. and Telgarsky, M. (2019).
\newblock The implicit bias of gradient descent on nonseparable data.
\newblock In Beygelzimer, A. and Hsu, D., editors, {\em Proceedings of the
  {Thirty}-{Second} {Conference} on {Learning} {Theory}}, volume~99 of {\em
  Proceedings of {Machine} {Learning} {Research}}, pages 1772--1798. PMLR.

\bibitem[Karush, 1939]{karush_minima_1939}
Karush, W. (1939).
\newblock {\em Minima of functions of several variables with inequalities as
  side conditions.}
\newblock {PhD} {Thesis}, Thesis (S.M.)–University of Chicago, Department of
  Mathematics, December 1939.

\bibitem[Kuhn and Tucker, 1951]{kuhn_nonlinear_1951}
Kuhn, H. and Tucker, A. (1951).
\newblock Nonlinear {Programming}.
\newblock In {\em Proceedings of the {Second} {Berkeley} {Symposium} on
  {Mathematical} {Statistics} and {Probability}}, pages 481--492. University of
  California Press.

\bibitem[Lakshminarayanan and Vikram~Singh, 2020]{lakshminarayanan_neural_2020}
Lakshminarayanan, C. and Vikram~Singh, A. (2020).
\newblock Neural {Path} {Features} and {Neural} {Path} {Kernel} :
  {Understanding} the role of gates in deep learning.
\newblock In Larochelle, H., Ranzato, M., Hadsell, R., Balcan, M.~F., and Lin,
  H., editors, {\em Advances in {Neural} {Information} {Processing} {Systems}},
  volume~33, pages 5227--5237. Curran Associates, Inc.

\bibitem[LeCun et~al., 2010]{lecun_mnist_2010}
LeCun, Y., Cortes, C., and Burges, C. (2010).
\newblock {MNIST} handwritten digit database.
\newblock {\em ATT Labs [Online]. Available: http://yann.lecun.com/exdb/mnist},
  2.

\bibitem[Lyu and Li, 2020]{lyu_gradient_2020}
Lyu, K. and Li, J. (2020).
\newblock Gradient {Descent} {Maximizes} the {Margin} of {Homogeneous} {Neural}
  {Networks}.
\newblock In {\em International {Conference} on {Learning} {Representations}}.

\bibitem[Mangasarian, 1999]{mangasarian_arbitrary-norm_1999}
Mangasarian, O.~L. (1999).
\newblock Arbitrary-norm separating plane.
\newblock {\em Operations Research Letters}, 24(1):15--23.

\bibitem[Nacson et~al., 2019]{nacson_lexicographic_2019}
Nacson, M.~S., Gunasekar, S., Lee, J., Srebro, N., and Soudry, D. (2019).
\newblock Lexicographic and {Depth}-{Sensitive} {Margins} in {Homogeneous} and
  {Non}-{Homogeneous} {Deep} {Models}.
\newblock In Chaudhuri, K. and Salakhutdinov, R., editors, {\em Proceedings of
  the 36th {International} {Conference} on {Machine} {Learning}}, volume~97 of
  {\em Proceedings of {Machine} {Learning} {Research}}, pages 4683--4692. PMLR.

\bibitem[Paszke et~al., 2019]{paszke_pytorch_2019}
Paszke, A., Gross, S., Massa, F., Lerer, A., Bradbury, J., Chanan, G., Killeen,
  T., Lin, Z., Gimelshein, N., Antiga, L., Desmaison, A., Kopf, A., Yang, E.,
  DeVito, Z., Raison, M., Tejani, A., Chilamkurthy, S., Steiner, B., Fang, L.,
  Bai, J., and Chintala, S. (2019).
\newblock {PyTorch}: {An} {Imperative} {Style}, {High}-{Performance} {Deep}
  {Learning} {Library}.
\newblock In Wallach, H., Larochelle, H., Beygelzimer, A., Alché-Buc, F.~d.,
  Fox, E., and Garnett, R., editors, {\em Advances in {Neural} {Information}
  {Processing} {Systems} 32}, pages 8024--8035. Curran Associates, Inc.

\bibitem[Pedersen, 2020]{pedersen_patchwork_2020}
Pedersen, T.~L. (2020).
\newblock {\em patchwork: {The} {Composer} of {Plots}}.

\bibitem[{R Core Team}, 2021]{r_core_team_r_2021}
{R Core Team} (2021).
\newblock {\em R: {A} {Language} and {Environment} for {Statistical}
  {Computing}}.
\newblock R Foundation for Statistical Computing, Vienna, Austria.

\bibitem[Saxe et~al., 2014]{saxe_exact_2014}
Saxe, A., McClelland, J., and Ganguli, S. (2014).
\newblock Exact solutions to the nonlinear dynamics of learning in deep linear
  neural networks.
\newblock International Conference on Learning Represenatations 2014.

\bibitem[Soudry et~al., 2018]{soudry_implicit_2018}
Soudry, D., Hoffer, E., Nacson, M.~S., Gunasekar, S., and Srebro, N. (2018).
\newblock The {Implicit} {Bias} of {Gradient} {Descent} on {Separable} {Data}.
\newblock {\em Journal of Machine Learning Research}, 19(70):1--57.

\bibitem[Stellato et~al., 2020]{stellato_osqp_2020}
Stellato, B., Banjac, G., Goulart, P., Bemporad, A., and Boyd, S. (2020).
\newblock {OSQP}: an operator splitting solver for quadratic programs.
\newblock {\em Mathematical Programming Computation}, 12(4):637--672.

\bibitem[Veness et~al., 2017]{veness_online_2017}
Veness, J., Lattimore, T., Bhoopchand, A., Grabska-Barwinska, A., Mattern, C.,
  and Toth, P. (2017).
\newblock Online {Learning} with {Gated} {Linear} {Networks}.
\newblock {\em arXiv}.
\newblock arXiv: 1712.01897.

\bibitem[Veness et~al., 2021]{veness_gated_2021}
Veness, J., Lattimore, T., Budden, D., Bhoopchand, A., Mattern, C.,
  Grabska-Barwinska, A., Sezener, E., Wang, J., Toth, P., Schmitt, S., and
  Hutter, M. (2021).
\newblock Gated {Linear} {Networks}.
\newblock {\em Proceedings of the AAAI Conference on Artificial Intelligence},
  35(11):10015--10023.

\bibitem[Wickham, 2016]{wickham_ggplot2_2016}
Wickham, H. (2016).
\newblock {\em ggplot2: {Elegant} {Graphics} for {Data} {Analysis}}.
\newblock Springer-Verlag New York.

\bibitem[Yuan and Lin, 2006]{yuan_model_2006}
Yuan, M. and Lin, Y. (2006).
\newblock Model selection and estimation in regression with grouped variables.
\newblock {\em Journal of the Royal Statistical Society: Series B (Statistical
  Methodology)}, 68(1):49--67.
\newblock \_eprint:
  https://rss.onlinelibrary.wiley.com/doi/pdf/10.1111/j.1467-9868.2005.00532.x.

\end{thebibliography}
\bibliographystyle{apalike}

\newpage
\appendix

\section{Experimental Setup}
\label{app:experimental-setup}

\subsection{Training the Finite-Time Predictors}

We train the models that use gradient descent with PyTorch \citep{paszke_pytorch_2019} and PyTorch Lightning (\url{https://github.com/PyTorchLightning/pytorch-lightning}) in Python 3.9.
We train all models for 1600 steps with a learning rate of 0.04 and for another 1600 steps with a learning rate of 0.01.

For the deep and shallow GLNs, we use orthogonal initialization \citep{saxe_exact_2014} for the weights and determine the contexts using hyperplanes, following \citet{veness_online_2017}.
In particular, we randomly sample hyperplanes $w\in\mathbb{R}^D$ in the input space using a normal distribution with mean $\mu=0$ and standard deviation $\sigma=36$.
We randomly sample a cutoff $b\in\mathbb{R}$ with mean $\mu=0$ and standard deviation $\sigma=9$.
The corresponding context function is then given by
$$
C(x)=
\begin{cases}
1&\text{ if }w^Tx-b\geq 0,\\
0&\text{ if }w^Tx-b<0.
\end{cases}
$$
To generate a context function with more than two possible contexts, we compose multiple hyperplanes, mapping each unique region produced by these multiple hyperplanes to its own context.
Unlike \citet{veness_online_2017}, we additionally use a cutoff $b$ that makes $C$ balanced, i.e. maps half the training data to $1$ and half the training data to $0$.
We use this median cutoff in the results presented in the main article.
\cref{app:ext-gln} discusses results on a random cutoff.

For the ReLU networks, we use Kaiming normal initialization \citep{he_delving_2015}.

\subsection{Convex Optimization}
\label{sec:cp}
To solve the convex optimization problems, we used the \texttt{cvxpy} library in Python \citep{diamond_cvxpy_2016,agrawal_rewriting_2018}, which follows the paradigm of Disciplined Convex Programming (DCP) \citep{grant_disciplined_2006}.
DCP follows a set of conventions on how to formalize convex optimization problems.
We trained the shallow GLN as well as the SVM-L2 and SVM-HL using the OSQP algorithm \citep{stellato_osqp_2020} with at most 10,000 iterations.
We trained the SVM-GLN, SVM-RC, and SVM-LC using the ECOS algorithm \citep{domahidi_ecos_2013} with at most 200 iterations.
We sampled the random contexts for the FReLUs using the same method as for the GLNs with median initialization.

Whereas the other architectures were easily translated into the DCP conventions, the SVM-L2 predictor required a bit more attention.
More specifically, it is more natural to express the architectural constraints by specifying $\zeta$ instead of $\beta$ and so we wanted to compute the equivalent of $\|\beta\|_2^2$ for $\zeta$.
It turns that this is given by

\begin{equation}
    \|A\zeta\|_2^2,\quad
    (A\zeta)_{hc}:=\frac{1}{HC}\sum_{h',c'}\zeta_{h'c'}-\sum_{c'}\zeta_{hc'}+C\zeta_{hc},
\end{equation}

which we prove below.

\begin{proof}
We have to prove that $\|\beta\|_2^2=\|A\zeta\|_2^2$.
Since $\beta$ satisfies our architectural constraints, we know that there is a $\zeta$ such that
$$
\beta_{\gamma}=\sum_{h=1}^H\zeta_{h\gamma_h}
$$
for all $\gamma$.
Thus,
\begin{align*}
    \|\beta\|_2^2&=\sum_{\gamma\in\{1,\dotsc,C\}^H}\left\|\sum_{h=1}^H\zeta_{h\gamma_h}\right\|_2^2=\sum_{\gamma\in\{1,\dotsc,C\}^H}\sum_{h,h'=1}^H\langle\zeta_{h\gamma_h},\zeta_{h'\gamma_{h'}}\rangle=\Delta_1+\Delta_2,
\end{align*}
where we define $\Delta_1$, $\Delta_2$ by changing the order of the summation operators and splitting the summation over $h$ and $h'$ into the case $h=h'$ (for $\Delta_1$) and $h\neq h'$ (for $\Delta_2$), i.e.
\begin{align*}
    \Delta_1=\sum_{h=1}^H\sum_{\gamma\in\{1,\dotsc,C\}^H}\|\zeta_{h\gamma_h}\|_2^2,\quad\Delta_2=\sum_{h\neq h'}\sum_{\gamma\in\{1,\dotsc,C\}^H}\langle\zeta_{h\gamma_h},\zeta_{h\gamma_h'}\rangle.
\end{align*}
We can now simplify these equations by noting that $\zeta_{h\gamma_h}$ is invariant to all but one dimension of our iterator $\gamma$:
\begin{align*}
    \Delta_1=C^{H-1}\sum_{h=1}^H\sum_{c=1}^C\|\zeta_{hc}\|_2^2,\quad
    \Delta_2=C^{H-2}\sum_{h\neq h'}\sum_{c,c'=1}^C\langle\zeta_{hc},\zeta_{h'c'}\rangle.
\end{align*}
Defining
\begin{align}
\begin{split}
    M\in\mathbb{R}^{(H\times C\times D)^2},\quad
    M_{\substack{hcd\\h'c'd'}}:=\delta_{dd'}\left(C\delta_{hh'}\delta_{cc'}+(1-\delta_{hh'})\right)=\delta_{dd'}\left(1+C(\delta_{hh'}\delta_{cc'})-\delta_{hh'}\right),
\end{split}
\end{align}
we can rewrite this norm as
\begin{align*}
    \|\beta\|_2^2=C^{H-2}\sum_{\substack{hcd\\h'c'd'}}M_{\substack{hcd\\h'c'd'}}\zeta_{hcd}\zeta_{h'c'd'}.
\end{align*}
All that remains to show is that $A^2=M$ (note that $A$ is symmetric).

To show this, we make a parameterized guess, following the ansatz
\begin{align*}
    \begin{split}
        A\in\mathbb{R}^{(H\times C\times D)^2},\quad
    A_{\substack{hcd\\h'c'd'}}=\delta_{dd'}\left(\alpha+\beta\delta_{hh'}+\gamma\delta_{cc'}+\kappa\delta_{hh'}\delta_{cc'}\right)
    \end{split}
\end{align*}
We thus require
\begin{align*}
    M_{\substack{hcd\\h'c'd'}}=\sum_{h'',c'',d''}A_{\substack{hcd\\h''c''d''}}A_{\substack{h''c''d''\\h'c'd'}}=\delta_{dd'}\left(HC\alpha+C\beta\delta_{hh'}+H\gamma\delta_{cc'}+\kappa\delta_{hh'}\delta_{cc'}\right),
\end{align*}
and therefore
\begin{equation}
    \alpha=\tfrac{1}{HC},
    \beta=-1,
    \gamma=0,
    \kappa=C.
\end{equation}
\end{proof}

\subsection{Data Analysis}

We performed all data analysis in R. \citep{r_core_team_r_2021}
All figures (except for \cref{fig:gln-example}a, for which we used Inkscape) were created using ggplot2 \citep{wickham_ggplot2_2016} and patchwork \citep{pedersen_patchwork_2020}.
In the supplementary material, we provide an R package that contains all data as well as the code reproducing all figures.

\section{Proofs}

\subsection{Proof of \cref{th}}
\label{app:proof-th}

We begin by restating the theorem:

\ghpt*

The theorem is operating under the following fundamental assumptions:

\begin{assumption}
\label{ass:exp-like}
$\ell$ is an exponential-like loss.
\end{assumption}

\begin{assumption}
\label{ass:loss-to-zero}
$\mathcal{L}(w^{(t)})\to0.$
\end{assumption}

\begin{assumption}
\label{ass:w-converges}
$w^{(t)}$ converges in direction to some $\hat{w}$.
\end{assumption}

In addition, we consider two assumptions excluding pathological cases, \cref{ass:pos-margin,ass:beta-gradient}, which we introduce at the point where they become necessary.

To prove the theorem, we must demonstrate (1) primal feasibility and (2) stationarity.

\textbf{Primal feasibility} simply involves proving that there is some $w^{\ast}=\alpha\hat{w}$ such that

$$
y^{(n)}f_{w^{\ast}}(x^{(n)};\gamma^{(n)})\geq1.
$$

This is fairly straightforward, with a minor complication being created by the fact that different polynomials $\mathcal{P}_{\gamma}$ may have different degrees $\nu_{\gamma}$.

For any given $\gamma$, we define the minimal margin of this context's data,
\begin{equation}
    m_{\gamma}:=\min_{n:\gamma^{(n)}=\gamma}y^{(n)}f_{\hat{w}}(x^{(n)};\gamma).
\end{equation}
By assumption, we know that $L(w^{(t)})\to 0$, and so we are guaranteed that for all $n$ and large enough $t$, $y^{(n)}f_{w^{(t)}}(x^{(n)})>0$.
However, we wish to exclude the pathological case, in which the margin of the normalized weight $\hat{w}^{(t)}$ still converges to $0$ instead of a positive value.
This motivates the following assumption:

\begin{assumption}
\label{ass:pos-margin}
For all $n$, $y^{(n)}f_{ \hat{w}}(x^{(n)};\gamma^{(n)})>0$.
\end{assumption}

From this assumption, we know that for all $\gamma$, $m_{\gamma}>0$.
Scaling $w$ by some $\alpha$ results in the changed margin

$$
y^{(n)}f_{\alpha\hat{w}}(x^{(n)};\gamma^{(n)})=\alpha^{\nu_{\gamma^{(n)}}}y^{(n)}f_{\hat{w}}(x^{(n)};\gamma^{(n)}).
$$

So for each $\gamma$, we can set the minimal margin to $1$ by scaling $w$ by $m_{\gamma}^{-1/\nu_{\gamma}}$.
Since we want to make sure that the margin for each context is not smaller than $1$, we scale $w$ by
\begin{equation}
    \alpha:=\max_{\gamma}m_{\gamma}^{-1/\nu_{\gamma}},
\end{equation}
setting $w^{\ast}:=\alpha\hat{w}$.
Thus, we know that $w^{\ast}$ satisfies the margin constraints (i.e. primal feasibility) and are left to check whether it also satisfies stationarity.

To demonstrate \textbf{stationarity}, we follow the same strategy as \citet{gunasekar_implicit_2018}.
To a large extent, we can make arguments that are exactly analogous to theirs and we refer to their proof in these cases.
For each $\gamma$, we consider the resulting sequence of linear predictors $\beta_{\gamma}^{(t)}=\mathcal{P}_{\gamma}(w^{(t)})$.
We would like to apply \cref{lem:gunasekar} (which allows us to consider any exponential-like loss, in contrast to \citet{gunasekar_implicit_2018}).
For this purpose, we consider the contextwise loss function
\begin{equation}
    \mathcal{L}_{\gamma}(\beta_{\gamma})=\sum_{n:\gamma^{(n)}=\gamma}\ell(y^{(n)}\langle\beta_{\gamma},x^{(n)}\rangle).
\end{equation}
Since the sum of all these loss functions converges to zero and all $\mathcal{L}_{\gamma}$ or nonnegative, we immediately know that $\mathcal{L}(\beta^{(t)})\to 0$.
Analogous to \citet{gunasekar_implicit_2018}, we also know that $\beta_{\gamma}^{(t)}/\|\beta^{(t)}_{\gamma}\|_2\to\hat{\beta}$, where
\begin{equation}
    \hat{\beta}:=\frac{\mathcal{P}(\hat{w})}{\left\|\mathcal{P}(\hat{w})\right\|}.
\end{equation}
All that is left is to exclude pathological cases where the gradient of the loss in the linear predictor does not converge:

\begin{assumption}
\label{ass:beta-gradient}
For all $\gamma$, $\nabla_{\beta}\mathcal{L}_{\gamma}(\beta_{\gamma}^{(t)})$ converges in direction to some $\hat{z}_{\gamma}$.
\end{assumption}

This allows us to apply \cref{lem:gunasekar} and infer that for each $\gamma$,
\begin{equation}
    \hat{z}_{\gamma}=\sum_{n\in S_{\gamma}}\lambda_ny_nx_n,\quad\lambda_n\geq0.
\end{equation}
The gradient update $\Delta w^{(t)}=\eta_t\nabla_w\mathcal{L}(w)$ can be decomposed into contextwise updates
\begin{equation}
    \Delta w^{(t)}_{\gamma}=\eta_t\nabla_w\mathcal{L}_{\gamma}(w).
\end{equation}
Similarly, we write
\begin{equation}
    w^{(t)}_{\gamma}:=w^{(0)}/|\Gamma|+\sum_{\gamma}\Delta w^{(t)}_{\gamma},
\end{equation}
which implies that
\begin{equation}
    w^{(t)}=\sum_{\gamma}w^{(t)}_{\gamma}
\end{equation}
(We could distribute the initial value $w^{(0)}=\sum_{\gamma\in\Gamma}w^{(0)}/|\Gamma|$ in different ways and only write it this way for convenience's sake. If $\Gamma$ is an infinite set, we leave away all empty contexts without loss of generalization.)
From \citet{gunasekar_implicit_2018}, equation (24), we know that we can write
\begin{equation}
    w^{(t)}_{\gamma}=\left(\nabla_w\mathcal{P}_{\gamma}(w)\hat{z}+\delta_{\gamma}^{(t)}\right)\sum_{u<t}\eta_up_{\gamma}(u)g(u)^{\nu_{\gamma}-1},
\end{equation}
where
\begin{equation}
    p(u)=\|z_{\gamma}^{(t)}\|_2,\quad
    g(u)=\|w^{(t)}\|_2,
\end{equation}
and $\delta_{\gamma}^{(t)}\to0$ is analogous to $\delta_3^{(t)}$ in their article.

For the purpose of a shorter notation, we now define
\begin{equation}
    k_{\gamma}^{(t)}:=\sum_{u<t}\eta_up_{\gamma}(u)g(u)^{\nu_{\gamma}-1},
\end{equation}
which means that
\begin{equation}
    w^{(t)}=\sum_{\gamma\in\Gamma}k^{(t)}_{\gamma}\left(\nabla_w\mathcal{P}_{\gamma}(w^{(\infty)})\hat{z}^{(\infty)}+\delta_{\gamma}^{(t)}\right).
\end{equation}
This $k^{(t)}_{\gamma}$ serves two purposes: it encodes the fact that $w^{(t)}$ diverges (since $k_{\gamma}^{(t)}$ diverges), and it specifies the scale of contributions to each contextwise gradient.
We thus disentangle these two purposes by defining the overall scale
\begin{equation}
    k^{(t)}:=\sum_{\gamma\in\Gamma}k_{\gamma}^{(t)},
\end{equation}
and the weights
\begin{equation}
    \hat{k}^{(t)}_{\gamma}:=k^{(t)}_{\gamma}/k^{(t)}.
\end{equation}
(Since all $k^{(t)}_{\gamma}$ diverge, we consider a large enough $t$ such that all $k_{\gamma}^{(t)}>0$.)
Defining
\begin{equation}
    \tilde{\delta}^{(t)}=\sum_{\gamma\in\Gamma}\hat{k}_{\gamma}^{(t)}\delta_{\gamma}^{(t)}\to 0,
\end{equation}
we can rewrite $w^{(t)}$ in a way that makes it more obvious how normalization will affect it:
\begin{equation}
    w^{(t)}=k^{(t)}\left(\sum_{\gamma}\hat{k}_{\gamma}^{(t)}\nabla_w\mathcal{P}_{\gamma}(w^{(\infty)})\hat{z}^{(\infty)}+\tilde{\delta}^{(t)}\right).
\end{equation}
We thus know that the normalized sequence of weights is given by
\begin{equation}
    \frac{w^{(t)}}{\|w^{(t)}\|}=\frac{\sum_{\gamma}\hat{k}_{\gamma}^{(t)}\nabla_w\mathcal{P}_{\gamma}(w^{(\infty)})\hat{z}^{(\infty)}+\tilde{\delta}^{(t)}}
    {\left\|\sum_{\gamma}\hat{k}_{\gamma}^{(t)}\nabla_w\mathcal{P}_{\gamma}(w^{(\infty)})\hat{z}^{(\infty)}+\tilde{\delta}^{(t)}\right\|}.
\end{equation}
We are left with two observation that will allow us to determine the limit of this equation and thus prove the theorem.
First, we can infer, analogous to Claim 1 in \citet{gunasekar_implicit_2018},
\begin{equation}
    \left\|\sum_{\gamma}\hat{k}_{\gamma}^{(t)}\nabla_w\mathcal{P}_{\gamma}(w^{(\infty)})\hat{z}^{(\infty)}\right\|>0,
\end{equation}
for large enough $t$ that all $\hat{k}_{\gamma}^{(t)}>0$.
Second, we must consider the limit of $\hat{k}_{\gamma}^{(t)}$.
Here, we face a small complication introduced by the context-gated setup: though we know that $\hat{k}_{\gamma}^{(t)}$ is upper bounded by $1$, we do not know whether it converges.
It is possible, for instance, that this sequence oscillates between two different values.
However, because the sequence is bounded, we know that it has a convergent subsequence $(t_s)_{s\in\mathbb{N}}$.
(For example, this subsequence may take into account only one of the two values between which $\hat{k}_{\gamma}^{(t)}$ may oscillate.)
We choose some bounded subsequence and define the limit
\begin{equation}
    \hat{k}_{\gamma}^{(t_s)}\to\hat{k}_{\gamma}.
\end{equation}
These two observations together allow us to infer that
\begin{equation}
    \hat{w}=\lim_{t\to\infty}\frac{w^{(t)}}{\|w^{(t)}\|}=\frac{\sum_{\gamma}\hat{k}_{\gamma}\nabla_w\mathcal{P}_{\gamma}(\hat{w})\hat{z}}
    {\left\|\sum_{\gamma}\hat{k}_{\gamma}\nabla_w\mathcal{P}_{\gamma}(\hat{w})\hat{z}\right\|}.
\end{equation}
This is clearly a linear sum of support vectors and since $w^{\ast}$ is a positive scaling of $\hat{w}$, it, too can be written as a linear sum of support vectors, proving the theorem.

\subsection{The Inductive Bias of GLNs with Two Layers}
\label{app:special-case}

\subsubsection{Architectural Constraints}
\label{app:special-case-arch}

We stated in the main article that architectural constraints are given by the fact that the difference between two predictors is invariant to the contexts they share, and that for a given architecture, we can choose exactly $(C-1)H+1$ linear predictors freely.
\cref{lem:arch-1} formalizes the first statement, \cref{lem:arch-2} the second one.

\begin{lemma}
\label{lem:arch-1}
Consider two pairs of context $\gamma^{(1)},\gamma^{(2)}$ and $\tilde{\gamma}^{(1)},\tilde{\gamma}^{(2)}$. We call this pair \emph{different only in contexts they share} if the following two conditions are true:

\begin{enumerate}
    \item If $\gamma^{(1)}_h=\gamma^{(2)}_h$, then $\tilde{\gamma}^{(1)}_h=\tilde{\gamma}^{(2)}_h$,
    \item If $\gamma^{(1)}_h\neq\gamma^{(2)}_h$, then $\gamma^{(i)}=\tilde{\gamma}^{(i)}$ for $i=1,2$.
\end{enumerate}

For any such pair of pairs,
\begin{equation}
\label{eq:arch-1}
    \beta_{\gamma^{(2)}}-\beta_{\gamma^{(1)}}=\beta_{\tilde{\gamma}^{(2)}}-\beta_{\tilde{\gamma}^{(1)}}.
\end{equation}
\end{lemma}

\begin{proof}
We know that
\begin{equation}
    \beta_{\gamma^{(2)}}-\beta_{\gamma^{(1)}}=
    \sum_{h=1}^Hw^{(2)}_h\left(w^{(1)}_{h\gamma^{(2)}_h}-w^{(1)}_{h\gamma^{(1)}_h}\right).
\end{equation}
If we prove that
\begin{equation}
    w^{(1)}_{h\gamma^{(2)}_h}-w^{(1)}_{h\gamma^{(1)}_h}=
    w^{(1)}_{h\tilde{\gamma}^{(2)}_h}-w^{(1)}_{h\tilde{\gamma}^{(1)}_h},
\end{equation}
for all $h$, we have proven the lemma.

To prove the equation, we consider two cases. If $\gamma^{(1)}_h=\gamma^{(2)}_h$, both sides of the equation are zero. If $\gamma^{(1)}_h\neq\gamma^{(2)}_h$, then $w^{(1)}_{h\gamma^{(i)}_h}=w^{(1)}_{h\tilde{\gamma}^{(i)}_h}$ for $i=1,2$ and again, the equation holds true.
\end{proof}

\begin{lemma}
\label{lem:arch-2}
Consider the set of contexts $\overline{\Gamma}$, where at most one hidden unit's local context is different from one, $\gamma_h\neq 1$.
We can pick $w$ to parameterize an arbitrary set of linear predictors $\beta_{\gamma}$ for all $\gamma\in\overline{\Gamma}$ and this set, in turn, uniquely determines all other linear predictors.
This means we can pick exactly $|\overline{\Gamma}|=(C-1)H+1$ linear predictors freely.
\end{lemma}

\begin{proof}
Consider an arbitrary set of predictors $\beta_{\gamma}$, $\gamma\in\overline{\Gamma}$.
Let us denote by $\gamma_1=(1)_{h=1,\dotsc,H}$ the vector where all contexts are $1$ and by $\gamma_{ch}=(1+(c-1)\delta_{h'h})_{h'=1,\dotsc,H}$ the vector where all entries are $1$ except for the $h$-th context, which is $c$. Any $\gamma\in\overline{\Gamma}$ can be expressed as either $\gamma_1$ or $\gamma_{hc}$.

We then define the weights $w^{(1)}\in\mathbb{R}^{C\times H\times D}$ as
\begin{align}
\begin{aligned}
    w^{(1)}_{11}&:=\beta_{\gamma_1},&\forall_{c>1}w^{(1)}_{c1}&:=\beta_{\gamma_{c1}},\\
    \forall_{h>1}w^{(1)}_{1h}&:=0,&
    \forall_{c>1}w^{(1)}_{ch}&:=\beta_{\gamma_{ch}}-\beta_{\gamma_1},
\end{aligned}
\end{align}
and $w^{(2)}=(1)_{h=1,\dotsc,H}$.

Using this definition, we can see that the predictor $\tilde{\beta}$ parameterized by $w$ is identical to $\beta$ for all $\gamma\in\overline{\Gamma}$:
\begin{align*}
    &\tilde{\beta}_{\gamma_1}=w^{(2)}_1w^{(1)}_{11}=\beta_{\gamma_1},
    \forall_{c>1}\tilde{\beta}_{\gamma_{c1}}=w^{(2)}_1w^{(1)}_{c1}=\beta_{\gamma_{c1}},\\
    &\forall_{h>1}\forall_{c>1}\tilde{\beta}_{\gamma_{ch}}=w_{11}^{(1)}+w_{ch}^{(1)}=\beta_{\gamma_{ch}}.
\end{align*}
We now show that, given this set of predictors, we can use (\ref{eq:arch-1}) to define all other predictors.
We use finite induction: consider some $h=1,\dotsc,H$ and suppose that we have already uniquely defined all $\beta_{\gamma}$ if $\gamma_k=1$ for all $k\geq h$ (for $h=1$, this is trivially true as $\gamma_1\in\overline{\Gamma}$).
For the induction, we must uniquely define any context $\gamma$ where $\gamma_h=c\neq 1$ and $\gamma_k=1$ for all $k\geq h+1$.

For any such context we consider the pair of contexts $\gamma,\gamma_{ch}$ and $\tilde{\gamma},\gamma_1$, where $\tilde{\gamma}_k=\gamma_k$ for all $k\neq h$ and $\tilde{\gamma}_h=1$. Since $\gamma_{ch},\gamma_1\in\overline{\Gamma}$, $\beta_{\gamma_{ch}},\beta_{\gamma_1}$ have already been defined. Since $\tilde{\gamma}_k=1$ for all $k\geq h$, $\beta_{\tilde{\gamma}}$ has been defined by the induction's assumption. Since these pairs are different only in contexts they share (namely only in dimension $h$), this immediately implies that
\begin{equation}
    \beta_{\gamma}=\beta_{\gamma_{ch}}+\beta_{\tilde{\gamma}}-\beta_{\gamma_1},
\end{equation}
is uniquely defined as well.
\end{proof}

\subsubsection{The GLN Norm}
\label{app:special-case-norm}

To prove (\ref{eq:beta-gln}), we prove the following statement:

\begin{proposition}
For a deep GLN with two layers, under \cref{ass:exp-like,ass:loss-to-zero,ass:w-converges,ass:pos-margin,ass:beta-gradient}, $\hat{\beta}$ is proportional to a minimum of
\begin{equation}
    \|\beta\|_{\rm GLN},\quad
    \text{s.t. }y^{(n)}\langle\beta_{\gamma^{(n)}},x^{(n)}\rangle\geq1,
\end{equation}
where
\begin{equation}
    \|\beta\|_{\rm GLN}=\min_{\zeta}\sum_{h=1}^H\|\zeta_h\|_2,\quad
    \text{s.t. }\beta_{\gamma}=\sum_{h=1}^H\zeta_{h\gamma_h}.
    \tag{\ref{eq:beta-gln} restated}
\end{equation}
\end{proposition}

\begin{proof}
Since $w$ converges in direction, $\zeta$ also converges in direction. We aim to prove that its limit direction, $\hat{\zeta}$, is proportional to a stationary point of
\begin{equation}
    \label{eq:opt-zeta}
    \sum_{h=1}^H\|\zeta_h\|_2,\quad\text{s.t. }y^{(n)}\left\langle\sum_{h=1}^H\zeta_{h\gamma_h},x^{(n)}\right\rangle\geq1
\end{equation}
Since (\ref{eq:opt-zeta}) is convex (because its objective function is convex and its constraints are linear), this will be sufficient to prove the proposition.

From \cref{th}, we know that $\hat{w}$ is proportional to a stationary point of (\ref{eq:th}) and can therefore be characterized as
\begin{align}
\begin{aligned}
\label{eq:kkt-w}
\hat{w}^{(1)}_{ch}=\hat{w}^{(2)}_h\sum_{\gamma:\gamma_h=c}\phi_{\gamma},\quad
\hat{w}^{(2)}_h=\sum_{c=1}^C\left\langle \hat{w}^{(1)}_{ch},\sum_{\gamma:\gamma_h=c}\phi_{\gamma}\right\rangle.
\end{aligned}
\end{align},

where $\phi_{\gamma}$ is the sum of support vectors for a given context:
\begin{equation}
    \phi_{\gamma}=\sum_{n\in S_{\gamma}}\lambda_nx^{(n)}y^{(n)}.
\end{equation}

$\hat{w}$ parameterizes $\hat{\zeta}$ as
\begin{equation}
    \hat{\zeta}_h=\hat{w}^{(2)}_h\hat{w}^{(1)}_{h}\in\mathbb{R}^{C\times D}
\end{equation}
A stationary point of (\ref{eq:opt-zeta}) is characterized by the equation
\begin{equation}
\label{eq:kkt-zeta}
    \zeta_{hc}=\|\zeta_h\|_2\sum_{\gamma:\gamma_h=c}\phi_{\gamma},
\end{equation}
so since the margin constraints are equivalent, we must only show that $\hat{\zeta}$ satisfies (\ref{eq:kkt-zeta}).

If $\hat{w}^{(2)}_h=0$, $\hat{\zeta_h}=0$ and (\ref{eq:kkt-zeta}) holds true.
We therefore assume that $\hat{w}^{(2)}_h\neq0$.
From (\ref{eq:kkt-w}), we can infer that

$$
\hat{w}_h^{(2)}=\hat{w}_h^{(2)}\sum_{c=1}^C\left\|\sum_{\gamma:\gamma_h=c}\phi_{\gamma}\right\|_2^2,
$$

which, since $\hat{w}_h^{(2)}\neq0$, implies that

\begin{equation}
\label{eq:support-id}
\sum_{c=1}^C\left\|\sum_{\gamma:\gamma_h=c}\phi_{\gamma}\right\|_2^2=1.
\end{equation}

That in turn implies that
\begin{align}
\begin{aligned}
\label{w-12-id}
\|\hat{w}^{(1)}_h\|_2^2=\sum_{c=1}^C\|\hat{w}^{(1)}_{hc}\|_2^2=
\left(\hat{w}^{(2)}_h\right)^2\sum_{c=1}^C\left\|\sum_{\gamma:\gamma_h=c}\phi_{\gamma}\right\|_2^2=\left(\hat{w}^{(2)}_h\right)^2,
\end{aligned}
\end{align}
and therefore
\begin{align*}
\hat{\zeta}_{hc}=\hat{w}^{(2)}_h\hat{w}^{(1)}_{hc}=\left(\hat{w}^{(2)}_h\right)^2\sum_{\gamma:\gamma_h=c}\phi_{\gamma}=
|\hat{w}_h^{(2)}|\|\hat{w}_h^{(1)}\|_2\sum_{\gamma:\gamma_h=c}\phi_{\gamma}=\|\hat{\zeta}_{h}\|_2\sum_{\gamma:\gamma_h=c}\phi_{\gamma},
\end{align*}
which proves (\ref{eq:kkt-zeta}).

To prove the reverse direction, we must show that if $\hat{\zeta}$ satisfies (\ref{eq:kkt-zeta}), we can find a $\hat{w}$ satisfying (\ref{eq:kkt-w}) such that $\hat{\zeta}_h=\hat{w}_h^{(2)}\hat{w}_h^{(1)}$.
If $\hat{\zeta}_h=0$, we define $\hat{w}_h^{(2)}=0$ and $\hat{w}_h^{(1)}=0$. We therefore assume $\hat{\zeta}_h\neq0$ and define
\begin{equation}
\label{zeta-to-w}
    \hat{w}^{(1)}_h:=\hat{\zeta}_h/\sqrt{\|\hat{\zeta}_h\|_2},\quad
    \hat{w}^{(2)}_h:=\sqrt{\|\hat{\zeta}_h\|_2}.
\end{equation}
Clearly this parameterizes $\zeta$ and so all that is left to show is that it satisfies (\ref{eq:kkt-w}).
(\ref{eq:kkt-zeta}) implies

\begin{align*}
\hat{w}_h^{(1)}=\hat{\zeta}_h/\sqrt{\|\hat{\zeta}_h\|_2}&=\sqrt{\|\hat{\zeta}_h\|_2}\sum_{\gamma:\gamma_h=c}\phi_{\gamma}=w^{(2)}_h\sum_{\gamma:\gamma_h=c}\phi_{\gamma}.
\end{align*}

Moreover, (\ref{eq:kkt-zeta}) immediately implies (\ref{eq:support-id}) and thus

\begin{align*}
\sum_{c=1}^C\left\langle \hat{w}^{(1)}_{hc},\sum_{\gamma:\gamma_h=c}\phi_{\gamma}\right\rangle&=
\frac{1}{\sqrt{\|\hat{\zeta}_h\|_2}}\sum_{c=1}^C\left\langle \hat{\zeta}_{hc},\sum_{\gamma:\gamma_h=c}\phi_{\gamma}\right\rangle=\\
\frac{1}{\sqrt{\|\hat{\zeta}_h\|_2}}\sum_{c=1}^C\left\langle \|\hat{\zeta}_{h}\|_2\sum_{\gamma:\gamma_h=c}\phi_{\gamma},\sum_{\gamma:\gamma_h=c}\phi_{\gamma}\right\rangle&=
\sqrt{\|\hat{\zeta}_h\|_2}\sum_{c=1}^C\left\|\sum_{\gamma:\gamma_h=c}\phi_{\gamma}\right\|_2^2
=\sqrt{\|\hat{\zeta}_h\|_2}=w^{(2)}_h,
\end{align*}

which proves (\ref{eq:kkt-w}) and thus the proposition.
\end{proof}
Note that we can prove the inductive bias of FreLUs (\ref{eq:frelu-norm}) in an analogous manner.

We now prove the following statement:

\begin{proposition}
For $C=H=2$, the norm is given by
\begin{align}
\begin{aligned}
\label{eq:beta-norm-3}
    \|\beta\|_{\rm GLN}^2=&\left(\|\beta_{11}-\beta_{12}\|_2+\|\beta_{11}-\beta_{21}\|_2\right)^2+\|\beta_{12}+\beta_{21}\|_2^2.
\end{aligned}
\end{align}
Equivalently,
\begin{equation}
    \|\beta\|_{\rm GLN}^2=\|\beta\|_2^2+\tfrac12\sum_{i,j}\|\beta_{ij}-\beta_{\overline{i}j}\|_2\|\beta_{ij}-\beta_{i\overline{j}}\|_2,
\tag{\ref{eq:beta-norm-2} restated}
\end{equation}
where $\overline{k}=2$ if $k=1$ and $\overline{k}=1$ if $k=2$.
\end{proposition}
\begin{proof}
We must minimize
\begin{equation}
    \min_{\zeta\in\mathbb{R}^{2\times 2\times D}}\|\zeta_1\|_2+\|\zeta_2\|_2,\quad
    \text{s.t. }\beta_{ij}=\zeta_{1i}+\zeta_{2j},
\end{equation}
for $i,j=1,2$.
For any $\zeta_{11}$, $\beta$ will uniquely determine $\zeta$ such that the constraint holds. More specifically,
\begin{align*}
    \zeta_{21}=\beta_{11}-\zeta_{11},\quad
    \zeta_{22}=\beta_{12}-\zeta_{11},\quad
    \zeta_{12}=\beta_{21}-\zeta_{21}=\beta_{21}-\beta_{11}+\zeta_{11}.
\end{align*}
Thus we must minimize the function
\begin{align}
\begin{aligned}
    g(\zeta_{11})&:=\|\zeta_1\|_2+\|\zeta_2\|_2=\sqrt{g_1(\zeta_{11})}+\sqrt{g_2(\zeta_{11})},\\
    g_1(\zeta_{11})&:=\|\zeta_{11}\|_2^2+\|\beta_{21}-\beta_{11}+\zeta_{11}\|_2^2,\\
    g_2(\zeta_{11})&:=\|\beta_{11}-\zeta_{11}\|_2^2+\|\beta_{12}-\zeta_{11}\|_2^2,
\end{aligned}
\end{align}
where we leave the dependence of $g$ on $\beta$ implicit to simplify the notation.
$g$ is convex and thus minimized by $\zeta_{11}$ if and only if
\begin{align*}
    0=\frac{\partial g}{\partial\zeta_{11}}(\zeta_{11}),
\end{align*}
which is equivalent to
\begin{align*}
    0
    =\sqrt{g_1(\zeta_{11})}\frac{\partial g_2}{\partial\zeta_{11}}(\zeta_{11})&+
    \sqrt{g_2(\zeta_{11})}\frac{\partial g_1}{\partial\zeta_{11}}(\zeta_{11})=\\
    \sqrt{g_1(\zeta_{11})}(2\zeta_{11}-\beta_{11}-\beta_{12})
    &+\sqrt{g_2(\zeta_{11})}(2\zeta_{11}+\beta_{21}-\beta_{11}).
\end{align*}
This means that
\begin{align*}
    2g(\zeta_{11})\zeta_{11}=\sqrt{g_1(\zeta_{11})}(\beta_{11}+\beta_{12})+\sqrt{g_2(\zeta_{11})}(\beta_{11}-\beta_{21}).
\end{align*}
Defining
\begin{equation}
    \alpha=\frac{\sqrt{g_1(\zeta_{11})}}{g(\zeta_{11})}\in[0,1],
\end{equation}
we can express
\begin{align*}
\zeta_{11}&=\tfrac12\left(\alpha(\beta_{12}+\beta_{11})+(1-\alpha)(\beta_{11}-\beta_{21})\right)=\tfrac12\left(\beta_{11}+\alpha\beta_{12}-(1-\alpha)\beta_{21}\right).
\end{align*}
We have therefore reduced our search space to $\alpha$, and reexpress the norm in terms of $\alpha$:
\begin{align}
\begin{aligned}
    \tilde{g}(\alpha):=&\sqrt{\tilde{g}_1(\alpha)}+\sqrt{\tilde{g}_2(\alpha)},\\
    \tilde{g}_1(\alpha):=&g_1(\zeta_{11})=\|\zeta_{11}\|_2^2+\|\beta_{21}-\beta_{11}+\zeta_{11}\|_2^2=\\
    &\tfrac14(\|\beta_{11}-\beta_{21}+\alpha(\beta_{12}+\beta_{21})\|_2^2+\|\beta_{21}-\beta_{11}+\alpha(\beta_{12}+\beta_{21})\|_2^2)=\\
    &\tfrac12(\|\beta_{11}-\beta_{21}\|_2^2+\alpha^2\|\beta_{12}+\beta_{21}\|_2^2),\\
    \tilde{g}_2(\alpha):=&g_2(\zeta_{11})=\|\beta_{11}-\zeta_{11}\|_2^2+\|\beta_{12}-\zeta_{11}\|_2^2=\\
    &\tfrac14(\|\beta_{11}+\beta_{21}-\alpha(\beta_{12}+\beta_{21}\|_2^2+
    \|(2-\alpha)(\beta_{12}+\beta_{21})-(\beta_{11}+\beta_{21})\|_2^2)=\\
    &\tfrac12(\|\beta_{11}+\beta_{21}\|_2^2+(1+(1-\alpha)^2)\|\beta_{12}+\beta_{21}\|_2^2-
    2\langle\beta_{12}+\beta_{21},\beta_{11}+\beta_{21}\rangle)=\\
    &\tfrac12(\|\beta_{11}-\beta_{12}\|_2^2+(1-\alpha)^2\|\beta_{12}+\beta_{21}\|_2^2).
\end{aligned}
\end{align}
We now wish to show that
\begin{equation}
\label{eq:beta-norm-3-zz}
\min_{\alpha}\tilde{g}(\alpha)=\sqrt{\tfrac12}\|\beta\|_{\rm GLN},
\end{equation}
according to the first definition, (\ref{eq:beta-norm-3}).
If $\beta_{12}+\beta_{21}=0$, this follows immediately.
We therefore assume that $\beta_{12}+\beta_{21}\neq0$.
Any minimal $\alpha$ must satisfy
\begin{equation}
\label{eq:stat-alpha}
    0=\frac{\partial\tilde{g}(\alpha)}{\partial\alpha}=\frac{\alpha\|\beta_{12}+\beta_{21}\|_2^2}{\sqrt{\tilde{g}_1(\alpha)}}-\frac{(1-\alpha)\|\beta_{12}+\beta_{21}\|_2^2}{\sqrt{\tilde{g}_2(\alpha)}}.
\end{equation}
Since $\alpha\in[0,1]$, both $\alpha$ and $1-\alpha$ are nonnegative, and (\ref{eq:stat-alpha}) is equivalent to
\begin{equation}
    \tilde{g}_1(\alpha)(1-\alpha)^2=\tilde{g}_2(\alpha)\alpha^2,
\end{equation}
which, in turn, reduces to
\begin{equation}
    \alpha^2\|\beta_{11}-\beta_{12}\|_2^2=(1-\alpha)^2\|\beta_{11}-\beta_{21}\|_2^2.
\end{equation}
Taking the square root and rearranging results in
\begin{equation}
    \alpha=\frac{\|\beta_{11}-\beta_{21}\|_2}{\|\beta_{11}-\beta_{12}\|_2+\|\beta_{11}-\beta_{21}\|_2}.
\end{equation}
This implies
\begin{align*}
\tilde{g}_1(\alpha)=\tfrac12\|\beta_{11}-\beta_{21}\|_2^2\left(1+\frac{\|\beta_{12}+\beta_{21}\|_2^2}{(\|\beta_{11}-\beta_{12}\|_2+\|\beta_{11}-\beta_{21}\|_2)^2}\right).
\end{align*}
Since
$$
1-\alpha=\frac{\|\beta_{11}-\beta_{12}\|_2}{\|\beta_{11}-\beta_{12}\|_2+\|\beta_{11}-\beta_{21}\|_2},
$$
we can analogously infer that
\begin{align*}
\tilde{g}_2(\alpha)=\tfrac12\|\beta_{11}-\beta_{12}\|_2^2\left(1+\frac{\|\beta_{12}+\beta_{21}\|_2^2}{(\|\beta_{11}-\beta_{12}\|_2+\|\beta_{11}-\beta_{21}\|_2)^2}\right).
\end{align*}
Finally, these identities imply that
\begin{align*}
    g_3(\alpha^{\ast})=
    &\sqrt{\tfrac12}(\|\beta_{11}-\beta_{12}\|_2+\|\beta_{11}-\beta_{21}\|_2)
    \sqrt{1+\frac{\|\beta_{12}+\beta_{21}\|_2^2}{(\|\beta_{11}-\beta_{12}\|_2+\|\beta_{11}-\beta_{21}\|_2)^2}}=\\
    &\sqrt{\tfrac12}\sqrt{(\|\beta_{11}-\beta_{12}\|_2+\|\beta_{11}-\beta_{21}\|_2)^2+\|\beta_{12}+\beta_{21}\|_2^2},
\end{align*}
which proves (\ref{eq:beta-norm-3-zz}) and thus (\ref{eq:beta-norm-3}).

All that is left to show is that (\ref{eq:beta-norm-3}) is equivalent to (\ref{eq:beta-norm-2}).
To do so, we note that
\begin{equation}
    \beta_{22}-\beta_{21}=\beta_{12}-\beta_{11}.
    \tag{\ref{eq:beta-id} recalled}
\end{equation}
Therefore
\begin{align}
    \begin{split}
        \|\beta\|_2^2=&\|\beta_{11}\|_2^2+\|\beta_{12}\|_2^2+\|\beta_{21}\|_2^2+\|\beta_{12}+\beta_{21}-\beta_{11}\|_2^2=\\
        &2(\|\beta_{11}\|_2^2+\|\beta_{12}\|_2^2+\|\beta_{21}\|_2^2+
        \langle\beta_{12},\beta_{21}\rangle-\langle\beta_{11},\beta_{12}\rangle-\langle\beta_{11},\beta_{21}\rangle
        ),
    \end{split}
\end{align}
and
\begin{align}
    \begin{split}
        &\|\beta_{11}-\beta_{12}\|_2^2+\|\beta_{11}-\beta_{21}\|_2^2+\|\beta_{12}+\beta_{21}\|_2^2=\\
        &2(\|\beta_{11}\|_2^2+\|\beta_{12}\|_2^2+\|\beta_{21}\|_2^2+
        \langle\beta_{12},\beta_{21}\rangle-\langle\beta_{11},\beta_{12}\rangle-\langle\beta_{11},\beta_{21}\rangle
        )=\|\beta\|_2^2.
    \end{split}
\end{align}
This, in turn implies that
\begin{align}
\begin{split}
        \|\beta\|^2=\|\beta\|_2^2+2\|\beta_{11}-\beta_{12}\|_2\|\beta_{11}-\beta_{21}\|_2
        =\|\beta\|_2^2+\tfrac12\sum_{i,j}\|\beta_{ij}-\beta_{\overline{i}j}\|_2\|\beta_{ij}-\beta_{i\overline{j}}\|_2,
\end{split}
\end{align}
where the latter equality follows from the fact that for all $i,j$
\begin{equation}
    \|\beta_{ij}-\beta_{\overline{i}j}\|_2\|\beta_{ij}-\beta_{i\overline{j}}\|_2=\|\beta_{11}-\beta_{12}\|_2\|\beta_{11}-\beta_{21}\|_2,
\end{equation}
as a result of (\ref{eq:beta-id}).
\end{proof}

\subsection{Extending \citet{gunasekar_characterizing_2018} to exponential-like losses}
\label{app:extension-gunasekar}

We here consider the same set of loss functions as \citet{soudry_implicit_2018} and \citet{lyu_gradient_2020}.
We call this class \emph{exponential-like losses}.

\begin{definition}
\label{def:exponential-like}
We call $\ell:\mathbb{R}\to\mathbb{R}$ \emph{exponential-like} if and only if the function satisfies the following assumptions:

\begin{enumerate}
    \item $\ell$ is monotonically decreasing, differentiable, and $\lim_{u\to\infty}\ell(u)\to0$.
    \item $\ell'$ is Lipschitz continuous.
    \item $-\ell'$ has a tight exponential tail, i.e. there exist positive constants $c$, $a$, $\mu_{+}$, $\mu_{-}$, $u_{+}$, $u_{-}$ such that
    \begin{align}
        \forall_{u>u_{+}}-\ell'(u)&\leq c(1+\exp(-\mu_{+}u))\exp(-au),\\
        \forall_{u>u_{-}}-\ell'(u)&\geq c(1-\exp(-\mu_{-}u))\exp(-au).
    \end{align}
\end{enumerate}
\end{definition}

\citet{gunasekar_characterizing_2018,gunasekar_implicit_2018} only consider the exponential loss, but note that they expect their results to generalize towards exponential-like losses.
More specifically, they prove Lemma 8 in \citet{gunasekar_characterizing_2018} only for the exponential loss.
We here close the small gap their results leave by extending this lemma to exponential-like losses.

\begin{lemma}
\label{lem:gunasekar}
For almost all linearly separable datasets $\mathcal{D}=\left((x^{(n)},y^{(n)})\right)_{n=1,\dotsc,N}$, consider the loss function
\begin{equation}
    \mathcal{L}(\beta):=\sum_{n=1}^N\ell(y^{(n)}\langle\beta,x^{(n)}\rangle),
\end{equation}
where $\ell$ is an exponential-like loss.

Any sequence $\beta^{(t)}$ such that
\begin{align}
    \mathcal{L}(\beta^{(t)})&\to 0,
    \label{eq:gun-conda-1}\\
    \beta^{(t)}/\|\beta^{(t)}\|_2&\to\hat{\beta},
    \label{eq:gun-conda-2}\\
    -\nabla_{\beta}\mathcal{L}(\beta^{(t)})/\|\nabla_{\beta}\mathcal{L}(\beta^{(t)})\|_2&\to \hat{z},
    \label{eq:gun-conda-3}
\end{align}
for some $\hat{\beta}$, $\hat{z}$. Let
\begin{equation}
    S_{\mathcal{D}}:=\left\{n|y^{(n)}\langle\hat{\beta},x^{(n)}\rangle=\min_ny^{(n)}\langle\hat{\beta},x^{(n)}\rangle\right\}
\end{equation}
be the \emph{support}. Then there exists a sequence of nonnegative numbers $(\alpha_n)_{n\in S_{\mathcal{D}}}$, $\alpha_n\geq 0$ such that

\begin{equation}
    \hat{z}=\sum_{n\in S_{\mathcal{D}}}\alpha_ny_nx_n.
\end{equation}
\end{lemma}

\begin{proof}
To prove the lemma, we reduce the general, exponential-like case to the special case where $\ell$ is the exponential loss.

The gradient is given by
\begin{equation}
    -\nabla_{\beta}\mathcal{L}(\beta)=\sum_{n=1}^N-\ell'(y^{(n)}\langle\beta,x^{(n)}\rangle)y^{(n)}x^{(n)}.
\end{equation}
We now decompose this gradient into the special case where $\ell$ is the exponential loss, and the deviation of $\ell$ from the exponential loss.
We write
\begin{equation}
    -\nabla_{\beta}\tilde{\mathcal{L}}(\beta)=\sum_{n=1}^N-c\exp(-ay^{(n)}\langle\beta,x^{(n)}\rangle)y^{(n)}x^{(n)},
\end{equation}
where $c$, $a$ are the constants from the definition of $\ell$.
(Put differently, $\tilde{\mathcal{L}}$ defines the dataset's loss if we were using the exponential loss.)

We know that
\begin{equation}
    m_n^{(t)}:=y^{(n)}\langle\beta^{(t)},x^{(n)}\rangle,
\end{equation}
diverges.
Thus we can assume a large enough $t$ such that $y^{(n)}\langle\beta^{(t)},x^{(n)}\rangle>u_+,u_-$ for all $n$.
We can now write
\begin{equation}
    -\ell'(m_n^{(t)})=c\exp(-am_n^{(t)})+c\exp(-am_n^{(t)})\delta_n^{(t)},
\end{equation}
where, due to the fact that $-\ell'$ has a tight exponential tail, we are guaranteed that
\begin{equation}
    -\exp(-\mu_-m_n^{(t)})\leq\delta_n^{(t)}\leq\exp(-\mu_+m_n^{(t)}),
\end{equation}
for large enough $t$.
We have thus sucessfully decomposed the gradient:
\begin{equation}
    -\nabla_{\beta}\mathcal{L}(\beta^{(t)})=-\nabla_{\beta}\tilde{\mathcal{L}}(\beta^{(t)})+c\sum_{n=1}^N\exp(-am_n^{(t)})\delta_n^{(t)}y^{(n)}x^{(n)}.
\end{equation}
We now want to use this decomposition to prove that
\begin{equation}
    -\nabla_{\beta}\tilde{\mathcal{L}}(\beta^{(t)})/\|\nabla_{\beta}\tilde{\mathcal{L}}(\beta^{(t)})\|_2\to \hat{z},
\end{equation}
as well.
We know that
\begin{equation}
    \delta^{(t)}:=\frac{\sum_{n=1}^N\exp(-am_n^{(t)})\delta_n^{(t)}y^{(n)}x^{(n)}}
    {\exp(-am_n^{(t)})y^{(n)}x^{(n)}}
\end{equation}
is bounded by $\tilde{\delta}^{(t)}:=\max_n\delta_n^{(t)}$, and therefore $\delta^{(t)}\to0$.
This, in turn, implies that
\begin{equation}
    (1+\delta^{(t)})/\|1+\delta^{(t)}\|_2\to1.
\end{equation}
We can rewrite
\begin{equation}
    -\nabla_{\beta}\mathcal{L}(\beta^{(t)})=-\nabla_{\beta}\tilde{\mathcal{L}}(\beta^{(t)})(1+\delta^{(t)}),
\end{equation}
and therefore infer
\begin{align}
\begin{split}
    \hat{z}
    =\lim_{t\to\infty}\frac{-\nabla_{\beta}\mathcal{L}(\beta^{(t)})}{\|\nabla_{\beta}\mathcal{L}(\beta^{(t)})\|}
    =\lim_{t\to\infty}\frac{-\nabla_{\beta}\tilde{\mathcal{L}}(\beta^{(t)})}{\|\nabla_{\beta}\tilde{\mathcal{L}}(\beta^{(t)})\|}\frac{1+\delta^{(t)}}{\|1+\delta^{(t)}\|}
    =\lim_{t\to\infty}\frac{-\nabla_{\beta}\tilde{\mathcal{L}}(\beta^{(t)})}{\|\nabla_{\beta}\tilde{\mathcal{L}}(\beta^{(t)})\|}
\end{split}
\end{align}
Clearly, $\tilde{\mathcal{L}}(\beta^{(t)})\to0$, which allows us to apply Lemma 8 from \citet{gunasekar_characterizing_2018} and proves the more general lemma.
\end{proof}

\section{Extended Data}

For clarity's sake, we only showed a subset of the data in \cref{fig:exp-gln,fig:exp-relu}. \cref{fig:supp-gln,fig:supp-relu} depict the full data from our experiments on GLNs and ReLU networks, respectively.

\subsection{Gated Linear Networks}
\label{app:ext-gln}

\begin{figure*}[t]
\vskip 0.2in
\begin{center}
\centerline{\includegraphics[width=\textwidth]{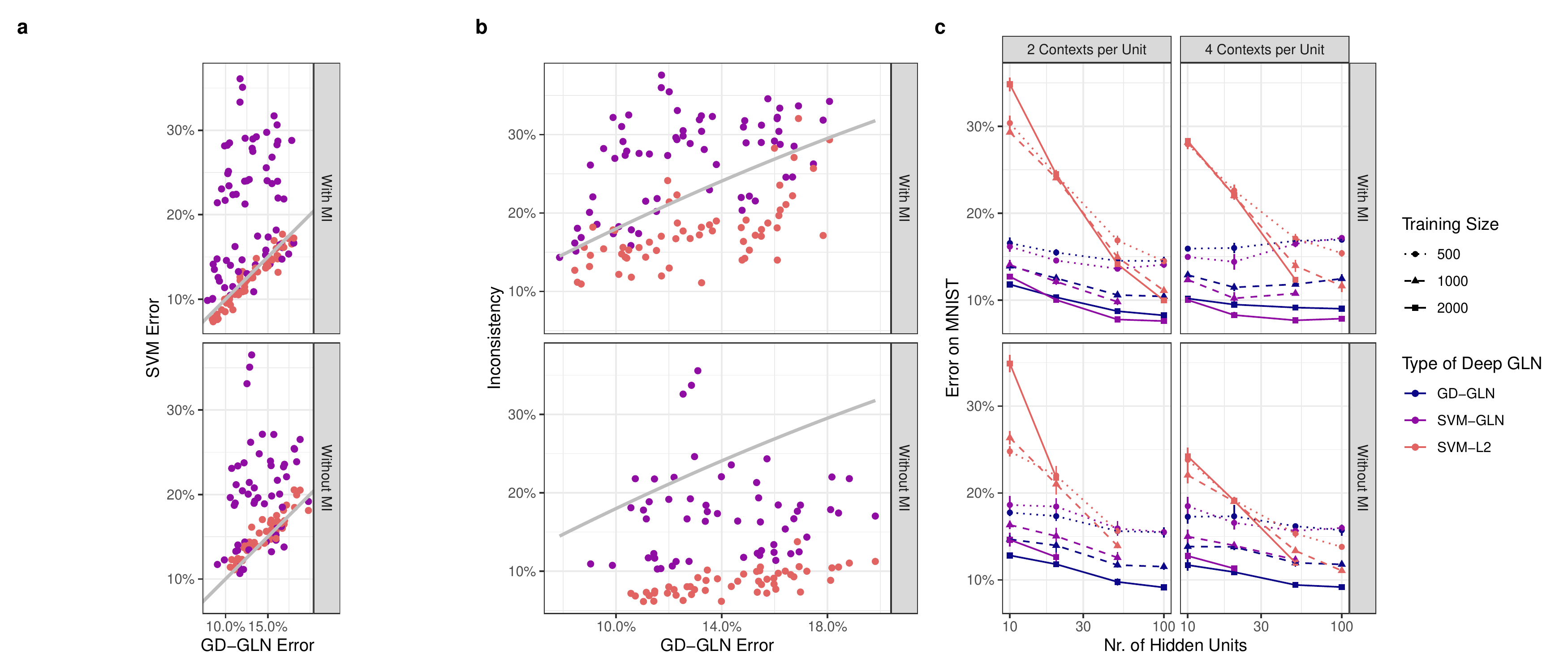}}
\caption{Full data from our experiments on GLNs. \textbf{a} The error of the GD-GLN plotted against the error of the SVMs. The grey line represents identical performance. The two panels correspond to the network that used median initialization ("With MI") and those that did not ("Without MI"). \textbf{b} Inconsistency between the GD-GLN and the SVM plotted against the GD-GLN's error on MNIST. The grey line represents the inconsistency we would expect from a predictor with matched error rate, but no further correlation with the network (see \cref{fn:inconsistency}). \textbf{c} Error of the SVMs and the GD-GLN plotted against the number of hidden units.}
\label{fig:supp-gln}
\end{center}
\vskip -0.2in
\end{figure*}

\begin{figure*}
\vskip 0.2in
\begin{center}
\centerline{\includegraphics[width=\textwidth]{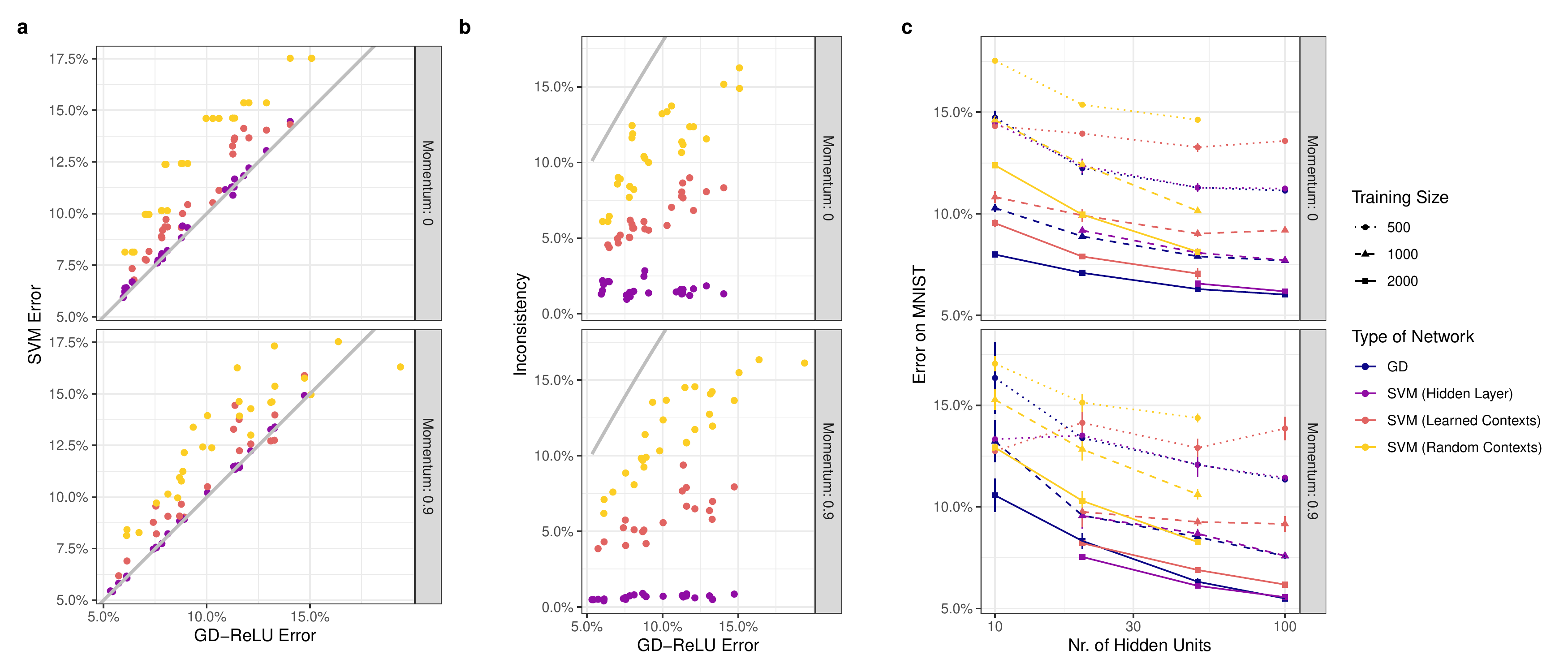}}
\caption{Full data from our experiments on ReLU networks. \textbf{a} The error of the GD-ReLU plotted against the SVMs. The grey line represents identical performance. The two panels correspond to training with and without momentum. \textbf{b} Inconsistency between the GD-ReLU and the SVMs plotted against the GD-ReLU's error on MNIST. The grey line represents the inconsistency we would expect from a predictor with matched error rate, but no further correlation with the network (see \cref{fn:inconsistency}). \textbf{c} Error of the SVMs and the GD-ReLU plotted against the number of hidden units.}
\label{fig:supp-relu}
\end{center}
\vskip -0.2in
\end{figure*}

Since \citet{veness_online_2017} did not use median initialization (MI), we trained the GLNs without MI as well.
This resulted in slightly worse performance, but the SVM-GLN was still more consistent with the GD-GLN than the SVM-L2 (\cref{fig:supp-gln}a,b).

In addition, \cref{fig:supp-gln}c, unlike \cref{fig:exp-gln}d, also depicts the networks with four contexts per unit.
Remarkably, for 100 hidden units, the SVM-L2 is beginning to outperform the SVM-GLN.
Other than that, the interpretation of the data remained unaffected.
Without MI, the data was qualitatively similar, as well, except that the GD-GLN tended to slightly outperform the SVM-GLN.

\subsection{Frozen-Gate ReLU Networks}

In \cref{fig:exp-relu}, we depicted the networks trained with a momentum of 0.9.
Since our theorem technically only holds for gradient descent without momentum (though \citet{soudry_implicit_2018} saw qualitatively similar behavior with momentum, as well), we additionally trained networks without momentum.
Since they do not rely on the ReLU networks at all, this did not change the SVM-RC.
It changed the SVM-LC and SVM-HL only insofar as the model from which they used the contexts and hidden layer, respectively, changed.
The interpretation of the data remained unaffected.
Most notably, both performance and consistency with the SVMs were a bit worse.
This is consistent with the interpretation that momentum speeds up training to a limit that is, in part, characterized by the SVMs, and that moving closer to this limit improves performance.

\end{document}